\documentclass[nohyperref]{article}

\usepackage{amsmath}
\usepackage{amsfonts}
\usepackage{bm}

\def\eqref#1{equation~\ref{#1}}
\def\1{\bm{1}}

\def\ra{{\textnormal{a}}}

\def\rr{{\textnormal{r}}}
\def\rs{{\textnormal{s}}}

\DeclareMathAlphabet{\mathsfit}{\encodingdefault}{\sfdefault}{m}{sl}
\SetMathAlphabet{\mathsfit}{bold}{\encodingdefault}{\sfdefault}{bx}{n}

\newcommand{\E}{\mathbb{E}}

\DeclareMathOperator*{\argmax}{arg\,max}

\usepackage{url}
\usepackage{subcaption}
\usepackage{algorithm}
\usepackage{algorithmic}
\usepackage{amssymb}
\usepackage{xcolor}         % colors
\usepackage{enumerate}
\usepackage{mathabx}

\usepackage{comment}
\usepackage{amsthm}
\usepackage{apxproof}
\usepackage{thmtools, thm-restate}
\usepackage{wrapfig}
\usepackage{adjustbox}
\usepackage[flushleft]{threeparttable}
\usepackage{graphicx}
\usepackage{tcolorbox}
\usepackage[page, header]{appendix}
\usepackage{titletoc}
\usepackage{lipsum}
\usepackage{hyperref}
\usepackage{multirow}
\usepackage{booktabs}
\usepackage{threeparttable}
\usepackage{enumitem}
\usepackage{mathabx}
\usepackage{subcaption}
\hypersetup{
    citecolor=blue,
    colorlinks=true,
    linkcolor=red,
    filecolor=magenta,      
    urlcolor=blue,
linktocpage}
\usepackage{comment}
\DeclareMathAlphabet\mathbfcal{OMS}{cmsy}{b}{n}

\usepackage{hyperref}

\usepackage[accepted]{icml2022}

\usepackage{amsmath}
\usepackage{amssymb}
\usepackage{mathtools}
\usepackage{amsthm}

\usepackage[capitalize,noabbrev]{cleveref}

\theoremstyle{plain}

\theoremstyle{definition}

\theoremstyle{remark}

\usepackage[textsize=tiny]{todonotes}

\icmltitlerunning{Mirror Learning}

\begin{document}

\twocolumn[
\icmltitle{Mirror Learning: A Unifying Framework of Policy Optimisation}



\icmlsetsymbol{equal}{*}

\begin{icmlauthorlist}
\icmlauthor{Jakub Grudzien Kuba}{ox}
\icmlauthor{Christian Schroeder de Witt}{ox}
\icmlauthor{Jakob Foerster}{ox}
\end{icmlauthorlist}

\icmlaffiliation{ox}{University of Oxford}

\icmlcorrespondingauthor{Jakub Grudzien Kuba}{jakub.grudzien@new.ox.ac.uk}

\icmlkeywords{Machine Learning, ICML}

\vskip 0.3in
]



\printAffiliationsAndNotice{}  % leave blank if no need to mention equal contribution

% While these latter are natural examples of mirror learning algorithms, GPI and TRL are pathologically restrictive or inpractical corner cases of our framework.

% This shows that TRPO and PPO are not the problem, but instead GPI and TRL are pathologically restrictive or inpractical corner cases of our framework.

% RL is GPI or TRL 
% But these suck in practice
% Implemented algos are very loosely (both theory end emp) connected to them
% Our mirror learning makes RL much bigger and explains TRPO and PPO
% We educate people
% Mirror learning is an algorithm FACTORY

\begin{abstract}
Modern deep reinforcement learning (RL) algorithms are \textit{motivated} by either the generalised policy iteration (GPI) or trust-region learning (TRL) frameworks. 
However, algorithms that \textit{strictly respect} these theoretical frameworks have proven unscalable.
Surprisingly, the only known scalable algorithms violate the GPI/TRL assumptions, e.g. due to required regularisation or other heuristics. The current explanation of their empirical success is essentially ``by analogy'': they are deemed approximate adaptations of theoretically sound methods. Unfortunately, studies have shown that in practice these algorithms differ greatly from their conceptual ancestors. 
In contrast, in this paper we introduce a novel theoretical framework, named \emph{Mirror Learning}, which provides theoretical guarantees to a large class of algorithms, including TRPO and PPO. While the latter two exploit the flexibility of our framework, GPI and TRL fit in merely as pathologically restrictive corner cases thereof. This suggests that the empirical performance of state-of-the-art methods is a direct consequence of their theoretical properties, rather than of aforementioned approximate analogies. 
Mirror learning sets us free to boldly explore novel, theoretically sound RL algorithms, a thus far uncharted wonderland.
% unconstrained by GPI and TRL . 

% Mirror learning liberates us from the limitations of prior frameworks and 
% provides us with an entire class of novel algorithms that come equipped with convergence guarantees.

% Excitingly, we show that mirror learning opens up a whole new space of policy learning methods equipped with convergence guarantees.
\end{abstract}
% PLAN
%1) Introduce GPI and algos derived from it (DONE)
%2) Introduce trust-region and TRPO and PPO (DONE)
%3) Introduce Mirror Learning: the concept, theorems, the DAG perspective (DONE)
%4) Say how other algos are instances of mirror learning (DONE)
%5) Mention experiments
\vspace{-15pt}
\section{Introduction}
The generalised policy iteration \citep[GPI]{sutton2018reinforcement} and trust-region learning \citep[TRL]{trpo} frameworks lay the foundations for the design of the most commonly used reinforcement learning (RL) algorithms. At each GPI iteration, an RL agent first evaluates its \textit{policy} by computing a scalar \textit{value function}, and then updates its policy so as to maximise the value function at every environment state. This procedure serves well for Markov Decision Problems \citep[MDPs]{bellman1957markov} with small state spaces, as it is guaranteed to produce a new policy whose value at every state improves monotonically over the old one. However, the sizes of state spaces that many practical problem settings constitute  are intractable to exact implementations of GPI. 
Instead, large scale settings employ function approximation and sample based learning. Unfortunately, such an adoption of GPI has proven unstable~\citep{mnih2015human}, which has necessitated a number of adjustments, such as replay buffers, target networks, etc., to stabilise learning \citep{van2016deep}. In their days, these heuristics have been empirically sucessful but were not backed up by any theory and required extensive hyperparameter tuning \citep{mnih2016asynchronous}. 
\begin{figure}[!ttbh]
    \centering
    \includegraphics[width=70mm]{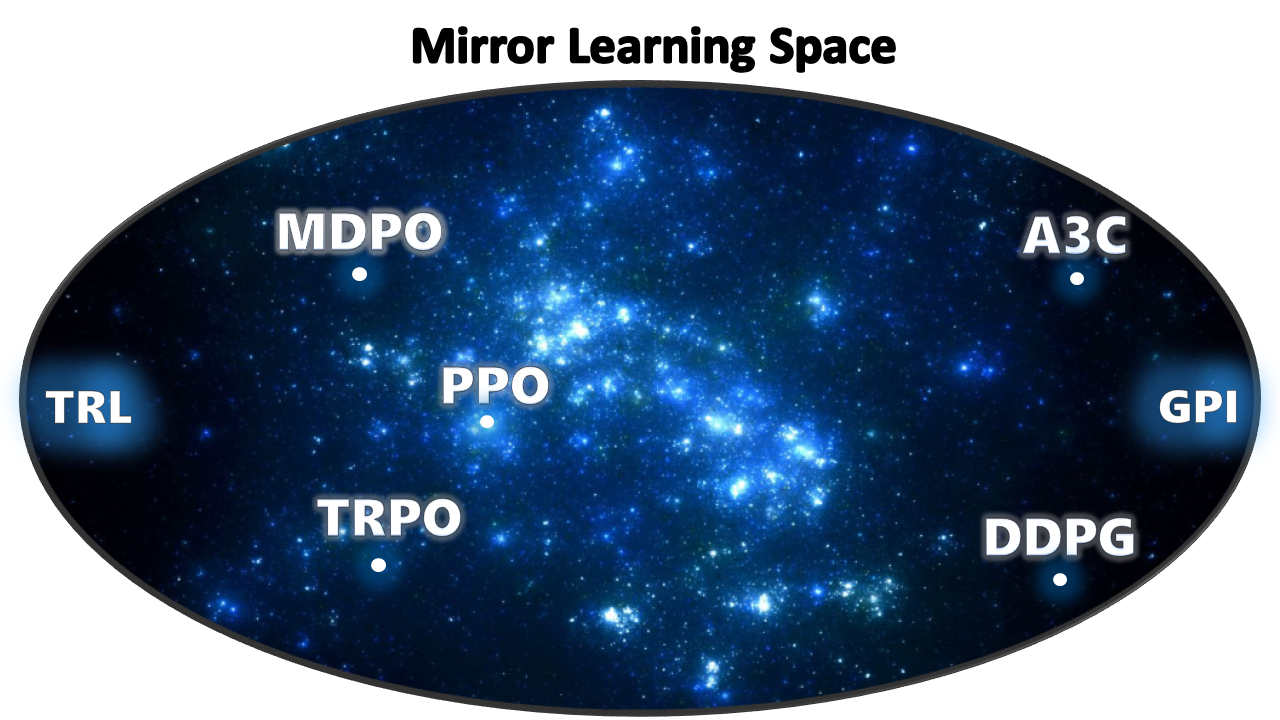}
    \vspace{-10pt}
    \caption{Known RL frameworks and algorithms as points in the infinite space of \textit{theoretically sound} mirror learning algorithms.}
    \vspace{-10pt}
    \label{fig:mirror-learning-galaxy}
\end{figure}
% To get around this issue, \textit{deep} RL algorithms \cite{sutton:nips12,silver2014deterministic} use neural network function approximators to generalise to states unaccessed during training. While, surprisingly, this approach sometimes works in practice \citep{mnih2015human}, it is not backed up by stringent theoretical guarantees and thus may readily break in non-trivial settings.

Instead, TRL improves the robustness of deep RL methods by optimising a \textit{surrogate} objective, which restricts the policy update size at each learning iteration while preserving monotonic improvement guarantees \citep{trpo}. To this end it introduces a notion of distance between policies, e.g. through evaluating the maximal KL-divergence. Unfortunately, these types of measures do not scale to large MDPs that TRL was meant to tackle in the first place (since small problems can be solved by GPI).

Nevertheless, TRL's theoretical promises and intuitive accessibility inspired a number of algorithms based on heuristic approximations thereof. This paradigm shift led to substantial empirical success: First, in \textit{Trust-Region Policy Optimization} \citep[TRPO]{trpo}, a hard constraint mechanism on the policy update size was introduced. Second, \textit{Proximal Policy Optimization} \citep[PPO]{ppo} replaced the hard constraint with the \textit{clipping objective} which, intuitively, disincentivises large policy updates. Since they were published, these algorithms have both been applied widely, resulting in state-of-the-art (SOTA) performance on a variety of benchmark, and even real-world, tasks \cite{trpo,ppo,berner2019dota}.
% apparent simplicity behind the conceptual narrative of the 

There is a stark contrast between the empirical success and the lack of theoretical understanding, which is largely limited to ``proof by analogy'': For example, PPO is regarded as \textit{sound} since it approximates an algorithm compliant with TRL. 
% Both this empirical success, as well as an overly superficial understanding of the TRL framework seem to have appeased concerns over the growing discrepancy between theory and practice. 
However, recent studies concluded that PPO arbitrarily exceeds policy update size constraints, thus fundamentally violating TRL principles \cite{wang2020truly, engstrom2020implementation}, which shows that PPO's empirical success does not follow from TRL. On a higher level, this reveals a concerning lack of theoretical understanding of the perhaps most widely used RL algorithm.

To reconnect theory with practice, in this paper we introduce a novel theoretical framework, named \emph{Mirror Learning}, which provides theoretical guarantees to a large class of known algorithms, including TRPO \cite{trpo}, PPO \citep{ppo}, and MDPO \citep{tomar2020mirror}, to name a few, as well as to myriads of algorithms that are yet to be discovered. Mirror learning is a general, principled policy-learning framework that possesses monotonic improvement and optimal-policy convergence guarantees, under arbitrary update-size constraints.
 Intuitively, a mirror-learning policy update maximises the current value function while keeping the, arbitrarily specified, update cost (called \textit{drift}) small. The update is further constrained within a \textit{neighbourhood} of the current policy, that can also be specified arbitrarily.

Since TRL and GPI were the only anchor points for theoretically sound policy optimisation, prior unsuccessful attempts to proving the soundness of PPO and other algorithms tried to shoe-horn them into the overly narrow confines of these frameworks~\citep{liu2019neural, wang2020truly, queeney2021generalized}. 
 Mirror learning shows that this was a flawed approach: Rather than shoe-horning the algorithms, we radically expand the space of theoretically sound methods, naturally covering PPO et al in their current, practical formulations. Our work suggests that the empirical performance of these state-of-the-art methods is a direct consequence of their theoretical properties, rather than of aforementioned approximate analogies. 
 Rather than being limited by the confines of two singular anchor points (GPI/TRL), mirror learning sets us free to explore an endless space of possible algorithms, each of which is already endowed with theoretical guarantees. 
 
We also illustrate the explanatory power of mirror learning and use it to explain a number of thus far \textit{unexplained} observations concerning the performance of other algorithms in the literature. 
%  Mirror learning sets us free to boldly explore novel, theoretically sound RL algorithms, a thus far uncharted wonderland.

Finally, we show that mirror learning allows us to view policy optimisation as a search on a directed graph, where the total path-weight between any two nodes in an optimisation path is upper bounded by the optimality gap between them, which we confirm experimentally. We also analyse proof-of-concept instantiations of mirror learning that use a variety of different neighbourhood and drift functions. 
%1) Explain RL, GPI, Trust-region, TRPO, PPO (DONE)
\section{Background}
\vspace{-5pt}
In this section, we introduce the RL problem formulation and briefly survey state-of-the-art learning protocols.
\vspace{-5pt}
\subsection{Preliminaries}
\vspace{-5pt}
We consider a Markov decision process (MDP) defined by a tuple $\langle \mathcal{S}, \mathcal{A}, r, P, \gamma, d\rangle$. Here, $\mathcal{S}$ is a discrete state space, $\mathcal{A}$ is a discrete action space\footnote{Our results extend to any compact state and action spaces. However, we work in the discrete setting in this paper for clarity.}, $r:\mathcal{S}\times \mathcal{A} \rightarrow [-R_{\max}, R_{\max}]$ is the bounded reward function, $P:\mathcal{S}\times\mathcal{A}\times\mathcal{S} \rightarrow [0, 1]$ is the probabilistic transition function, $\gamma \in [0, 1)$ is the discount factor, and $d\in\mathcal{P}(\mathcal{S})$ (here $\mathcal{P}(X)$ denotes the set of probability distributions over a set $X$) is the initial state distribution. At time step $t\in\mathbb{N}$, the agent is at a state $\rs_t$, takes an action $\ra_t$ according to its stationary policy $\pi(\cdot|\rs_t)\in\mathcal{P}(\mathcal{A})$, receives a reward $r(\rs_t, \ra_t)$, and moves to the state $\rs_{t+1}$, whose probability distribution is $P(\cdot|\rs_t, \ra_t)$. The whole experience is evaluated by the \textit{return}, which is the discounted sum of all rewards
\begin{align}
\vspace{-5pt}
    \rr^{\gamma}_{\infty} \triangleq  \sum\limits_{t=0}^{\infty}\gamma^t r(\rs_t, \ra_t).\nonumber
\end{align}
The state-action and state value functions that evaluate the quality of states and actions, by providing a proxy to the expected return, are given by
\begin{align}
\vspace{-5pt}
    &\quad Q_{\pi}(s, a) \triangleq \E_{\rs_{1:\infty}\sim P, \ra_{1:\infty}\sim\pi}\big[ \rr^{\gamma}_{\infty} \ | \ \rs_0 = s, \ \ra_0 = a\big], \nonumber\\
    & \quad \quad V_{\pi}(s) \triangleq \E_{\ra_0\sim\pi, \rs_{1:\infty}\sim P, \ra_{1:\infty}\sim\pi}\big[ \rr^{\gamma}_{\infty} \ | \ \rs_0 = s \big], \nonumber
\end{align}
respectively. The \textit{advantage function}, defined as
\begin{align}
\vspace{-5pt}
    A_{\pi}(s, a) \triangleq Q_{\pi}(s, a) - V_{\pi}(s),\nonumber
\end{align}
estimates the advantage of selecting one action over another.
The goal of the agent is to learn a policy that maximises the \textit{expected return}, defined as a function of $\pi$,
\begin{align}
    \vspace{-5pt}
    \eta(\pi) \triangleq \E_{\rs_0\sim d, \ra_{0:\infty}\sim\pi, \rs_{0:\infty}\sim P}\big[\rr^{\gamma}_{\infty}\big] = \E_{\rs\sim\rho_{\pi}, \ra\sim\pi}\big[ r(\rs, \ra) \big].\nonumber
\end{align}
Here, $\rho_{\pi}(s)\triangleq \sum\limits_{t=0}^{\infty}\gamma^t \text{Pr}(\rs_t=s|\pi)$ is the (improper) marginal state distribution.
Interestingly, the set of solutions to this problem always contains the \textit{optimal policies}---policies $\pi^*$, for which 
$Q_{\pi^*}(s, a) \triangleq Q^{*}(s, a) \geq Q_{\pi}(s, a)$ holds for any $\pi\in\Pi\triangleq\bigtimes_{s\in\mathcal{S}}\mathcal{P}(\mathcal{A})$, $s\in\mathcal{S}, a\in\mathcal{A}$. Furthermore, the optimal policies are the only ones that satisfy the \textit{optimality equation} \citep{sutton2018reinforcement}
\begin{align}
    \vspace{-5pt}
    \pi^*(\cdot|s) = \argmax_{\bar{\pi}(\cdot|s)\in \mathcal{P}(\mathcal{A})}\E_{\ra\sim\bar{\pi}}\big[Q_{\pi^*}(s, \ra)\big], \ \forall s\in\mathcal{S}. \nonumber
\end{align}
In the next subsections, we survey the most fundamental and popular approaches of finding these optimal policies.

\subsection{Generalised Policy Iteration}
\vspace{-5pt}
One key advantage of the \textit{generalised policy iteration} (GPI) framework is its simplicity. Even though the policy influences both the rewards and state visitations, GPI guarantees that simply reacting \textit{greedily} to the value function, 
\begin{align}
\vspace{-5pt}
    \label{eq:gpi}
    \pi_{\text{new}}(\cdot|s) = \argmax_{\bar{\pi}(\cdot|s)\in\mathcal{P}(\mathcal{A})} \E_{\ra\sim\bar{\pi}}\big[ Q_{\pi_{\text{old}}}(s, \ra)\big], \ \forall s\in\mathcal{S}
\end{align}
guarantees that the new policy obtains higher expected returns at every state, \emph{i.e.,} $V_{\pi_{\text{new}}}(s)\geq V_{\pi_{\text{old}}}(s), \ \forall s\in\mathcal{S}$. Moreover, this procedure converges to the set of optimal policies \citep{sutton2018reinforcement}, which can be seen intuitively by substituting a fixed-point policy into Equation (\ref{eq:gpi}). 

In settings with small, discrete action spaces GPI can be executed approximately without storing the policy variable $\pi$ by responding greedily to the state-action value function. This gives rise to \textit{value-based learning} \citep{sutton2018reinforcement}, where the agent learns a Q-function with the Bellman-max update
\begin{align}
\vspace{-5pt}
    Q_{\text{new}}(s, a) = r(s, a) + \gamma\cdot \E_{\rs'\sim P}\big[ \max_{a'\in\mathcal{A}}Q_{\text{old}}(\rs', a') \big], \nonumber
\end{align}
which is known to converge to $Q^{*}$ \citep{Watkins1992}, and has inspired design of a number of methods \citep{mnih2015human, van2016deep}.
%This approach has been implemented, for example, in the \textit{Deep Q-Network} (DQN) algorithm \citep{mnih2015human}, where a neural network learns the state-action value function by minimising the (empirical) error between the left- and right-hand sides of Equation (\ref{eq:q-learning-update}). 
%Although our understanding of the interaction between deep learning and RL is still lacking, it has recently been proved that, in the regime of infinite-width neural networks, such an algorithm converges to the optimal state-action value function, $Q^*$ \citep{fan2020theoretical}. 

Another, approximate, implementation of GPI is through the policy gradient (PG) algorithms \citep{sutton:nips12}. These are methods which optimise the policy $\pi_\theta$ by parameterising it with $\theta$, and updating the parameter in the direction of the gradient of the expectd return, given by
\begin{align}
\vspace{-5pt}
    \nabla_{\theta}\eta(\pi_\theta)|_{\theta=\theta_{\text{old}}}=\E_{\rs\sim\rho_{\pi_{\theta_{\text{old}}}}}\Big[ \nabla_{\theta}\E_{\ra\sim\pi_{\theta}}\big[ Q_{\pi_{\theta_{\text{old}}}}(\rs, \ra) \big]\big|_{\theta=\theta_{\text{old}}} \Big],\nonumber
\end{align}
which is the gradient of the optimisation objective of GPI from Equation (\ref{eq:gpi}) weighted by $\rho_{\pi_{\theta_{\text{old}}}}$. An analogous result holds for (continuous) deterministic policies \citep{silver2014deterministic, lillicrap2015continuous}. Thus, PG based algorithms approximately solve the GPI step, with a policy in the neighbourhood of $\pi_{\theta_{\text{old}}}$, provided the step-size $\alpha>0$ in the update $\theta_{\text{new}} = \theta + \alpha \nabla_{\theta}\eta(\pi_{\theta})|_{\theta=\theta_{\text{old}}}$ is sufficiently small. PG methods have played a major role in applications of RL to the real-world settings, and especially those that involve continuous actions, where some value-based algorithms, like DQN, are intractable \citep{williams1992simple, baxter2001infinite, mnih2016asynchronous}.
\vspace{-5pt}
\subsection{Trust-Region Learning}
\vspace{-5pt}
In practice, policy gradient methods may suffer from the high variance of the PG estimates and training instability \citep{kakade2002approximately, sugiyama}. \textit{Trust-region learning} (TRL) is a framework that aims to solve these issues. At its core lies the following policy update
\begin{align}
\vspace{-5pt}
    \label{eq:trust-region}
    &\pi_{\text{new}} = \argmax_{\bar{\pi}\in\Pi}\E_{\rs\sim\rho_{\pi_{\text{old}}}, \ra\sim\bar{\pi}}\big[ A_{\pi_{\text{old}}}(\rs, \ra) \big] - C\text{D}_{\text{KL}}^{\text{max}}(\pi_{\text{old}}, \bar{\pi}),\nonumber\\
    &\quad \quad \quad \quad \text{where }C = 4\gamma \max_{s, a}|A_{\pi_{\text{old}}}(s, a)|/(1-\gamma)^2.
\end{align}
It was shown by \citet{trpo} that this update guarantees the monotonic improvement of the return, \emph{i.e.,} $\eta(\pi_{\text{new}}) \geq \eta(\pi_{\text{old}})$. Furthermore, the KL-penalty in the above objective ensures that the new policy stays within the neighbourhood of $\pi_{\text{old}}$, referred to as the \textit{trust region}. This is particularly important with regard to the instability issue of the PG-based algorithms \citep{duan2016benchmarking}.

Although the exact calculation of the max-KL penalty is intractable in settings with large/continuous state spaces, the algorithm can be heuristically approximated, e.g. through \textit{Trust Region Policy Optimization} \citep[TRPO]{trpo}, which performs a constrained optimisation update
\begin{align}
\vspace{-5pt}
    &\quad \quad \pi_{\text{new}} = \argmax_{\bar{\pi}\in\Pi}\E_{\rs\sim\rho_{\pi_{\text{old}}}, \ra\sim\bar{\pi}}\big[ A_{\pi_{\text{old}}}(\rs, \ra) \big] \nonumber\\
    &\text{subject to }\E_{\rs\sim\rho_{\pi_{\text{old}}}}\big[ \text{D}_{\text{KL}}\big(\pi_{\text{old}}(\cdot|\rs), \pi_{\text{new}}(\cdot|\rs) \big) \big] \leq \delta,\nonumber 
\end{align}
where $\delta>0$.  Despite its deviation from the original theory, empirical results suggest that TRPO approximately maintains the TRL properties. However, in order to achieve a simpler TRL-based heuristic, \citet{ppo} introduced \textit{Proximal Policy Optimization} (PPO) which updates the policy by
\begin{align}
\vspace{-5pt}
    \label{eq:ppo}
    &\quad \quad \quad \quad \pi_{\text{new}} = \argmax_{\bar{\pi}\in\Pi} \E_{\rs\sim\rho_{\pi_{\text{old}}}, \ra\sim\pi_{\text{old}}}\big[ L^{\text{clip}}\big],\\
    &L^{\text{clip}}= \min\Big( \rr(\bar{\pi}) A_{\pi_{\text{old}}}(s, a), \text{clip}\big( \rr(\bar{\pi}), 1\pm\epsilon\big)A_{\pi_{\text{old}}}(s, a)\Big),\nonumber
\end{align}
where for given $s, a$, $\rr(\bar{\pi})\triangleq \bar{\pi}(a|s)/\pi(a|s)$. 
The \textit{clip} operator truncates $\rr(\bar{\pi})$ to $1-\epsilon$ (or $1+\epsilon$), if it is below (or above) the threshold interval. Despite being motivated from a TRL perspective, PPO violates its very core principles by failing to constrain the update size---whether measured either by KL divergence, or by the likelihood ratios \citep{wang2020truly, engstrom2020implementation}. Nevertheless, PPO's ability to stabilise policy training has been demonstrated on a wide variety of tasks, frequently resulting in SOTA performance \citep{henderson2018deep, berner2019dota}. This begs the question how PPO's empirical performance can be justified theoretically. Mirror learning, which we introduce in the next section, addresses this issue.
\vspace{-10pt}
\section{Mirror Learning}
\vspace{-5pt}
In this section we introduce \textit{Mirror Learning}, state its theoretical properties, and explore its connection to prior RL theory and methodology.
\vspace{-5pt}
\subsection{The Framework}
\vspace{-5pt}
We start from the following definition.
\begin{restatable}{definition}{definedrift}
\label{def:drift}
A \textit{drift functional} $\mathfrak{D}$ is a map 
\begin{align}
    \mathfrak{D}:\Pi\times\mathcal{S} \rightarrow \{ \mathfrak{D}_{\pi}(\cdot|s):
    \mathcal{P}(\mathcal{A}) \rightarrow  \mathbb{R} \}, \nonumber
\end{align}
such that for all $s\in\mathcal{S}$, and $\pi, \bar{\pi}\in\Pi$, writing $\mathfrak{D}_{\pi}(\bar{\pi}|s)$ for $\mathfrak{D}_{\pi}\big(\bar{\pi}(\cdot|s)|s\big)$, the following conditions are met
\begin{enumerate}
    \vspace{-5pt}
    \item  \label{bullet:drif2}$\mathfrak{D}_{\pi}(\bar{\pi}|s)\geq \mathfrak{D}_{\pi}(\pi|s) = 0$ \textit{(nonnegativity)},
    \item  \label{bullet:drif3}$\mathfrak{D}_{\pi}(\bar{\pi}|s)$ has zero gradient\footnote{More precisely, all its G\^ateaux derivatives are zero.} with respect to $\bar{\pi}(\cdot|s)$, evaluated at $\bar{\pi}(\cdot|s)=\pi(\cdot|s)$ \textit{(zero gradient)}.
\end{enumerate}
\vspace{-5pt}
Let $\nu^{\bar{\pi}}_{\pi}(s) \in \mathcal{P}(\mathcal{S})$ be a state distribution that can depend on $\pi$ and $\bar{\pi}$.
The \textit{drift} $\mathfrak{D}^{\nu}$ of $\bar{\pi}$ from $\pi$ is defined as
\begin{align}
\vspace{-5pt}
    \mathfrak{D}^{\nu}_{\pi}(\bar{\pi}) \triangleq \E_{\rs\sim \nu_{\pi}^{ \bar{\pi}}}\big[ \mathfrak{D}_{\pi}(\bar{\pi}|\rs) \big]\nonumber
\end{align}
where $\nu^{\pi}_{\bar{\pi}}$ is such that the expectation is continuous in $\pi$ and $\bar{\pi}$ and monotonically non-decreasing in the total variation distance between them.
The drift is \textit{positive} if $\mathfrak{D}^{\nu}_{\pi}(\bar{\pi}) = 0$ implies $\bar{\pi} = \pi$, and \textit{trivial} if $\mathfrak{D}^{\nu}_{\pi}(\bar{\pi}) = 0$, $\forall \pi, \bar{\pi}\in\Pi$.
\end{restatable}
We would like to highlight that drift is not the Bregman distance, associated with mirror descent \citep{nemirovskij1983problem, beck2003mirror}, as we do not require it to be (strongly) convex, or differentiable everywhere. All we require are, the much milder, continuity and G{\^ a}teaux-differentiability \citep{gateaux1922diverses, hamilton1982global} of the drift functional at $\bar{\pi}(\cdot|s) = \pi(\cdot|s)$.

We introduce one more concept whose role is to generally account for explicit update-size constraints. Such constraints reflect a learner's risk-aversity towards subtantial changes to its behaviour, like in TRPO, or are dictated by an algorithm design, like the learning rate in PG methods. 
\begin{restatable}{definition}{neighbourhood}
We say that $\mathcal{N}:\Pi\rightarrow\mathbb{P}(\Pi)$ is a \textit{neighbourhood operator}, where $\mathbb{P}(\Pi)$ is the power set of $\Pi$, if
\begin{enumerate}
\vspace{-5pt}
    \item It is a continuous map \textit{(continuity)},
    \item Every $\mathcal{N}(\pi)$ is a compact set \textit{ (compactness)},
    \item  There exists a metric $\chi:\Pi\times \Pi\rightarrow\mathbb{R}$, such that $\forall \pi \in \Pi$, there exists $\zeta>0$, such that $\chi(\pi, \bar{\pi})\leq \zeta$ implies $\bar{\pi}\in\mathcal{N}(\pi)$ \textit{(closed ball)}.
\end{enumerate}
\vspace{-5pt}
 The \textit{trivial} neighbourhood operator is $\mathcal{N}\equiv \Pi$.
\end{restatable}

Let $\mathfrak{D}^{\nu}$ be a drift, and $\pi, \bar{\pi}$ be policies. Suppose that $\beta_{\pi}\in\mathcal{P}(\mathcal{S})$ is a state distribution, referred to as a \textit{sampling distribution}, such that $\beta_{\pi}(s) > 0, \forall s\in\mathcal{S}$. Then, the \textit{mirror operator} transforms the value function of $\pi$ into the following functional of $\bar{\pi}$,
\begin{align}
\vspace{-5pt}
\hspace{-10pt}
    %[\mathcal{M}^{\bar{\pi}}_{\mathfrak{D}} V_{\pi}](s) = \E_{\ra\sim\bar{\pi}}\big[ Q_{\pi}(s, \ra)
    %\big] - \frac{\nu^{\bar{\pi}}_{\pi} (s)}{\beta_{\pi}(s)}%
    %\mathfrak{D}_{\pi}(\bar{\pi}|s). \nonumber
    [\mathcal{M}^{\pi_{\text{new}}}_{\mathfrak{D}} Q_{\pi_{\text{old}}}](s) = \E_{\ra\sim\pi_{\text{new}}}\big[ Q_{\pi_{\text{old}}}(s, \ra)
    \big] -
\mathfrak{D}_{\pi_{\text{old}}}(\pi_{\text{new}}|s)
    \vspace{-5pt}\nonumber
\end{align}
Note that, in the above definition, $Q_{\pi}$ can be equivalently replaced by $A_{\pi}$, as this only subtracts a constant $V_{\pi}(s)$ independent of $\bar{\pi}$. We will use this fact later.
As it turns out, despite the appearance of the drift penalty, simply acting to increase the mirror operator suffices to guarantee the policy improvement, as summarised by the following lemma, proved in Appendix \ref{appendix:lemmas}.
\begin{restatable}{lemma}{mpi}
\label{lemma:mpi}
Let $\pi_{\text{old}}$ and $\pi_{\text{new}}$ be policies. Suppose that
\begin{align}   
    \label{ineq:what-mirror-step-does}
    [\mathcal{M}^{\pi_{\text{new}}}_{\mathfrak{D}} V_{\pi_{\text{old}}}](s) \geq 
    [\mathcal{M}^{\pi_{\text{old}}}_{\mathfrak{D}} V_{\pi_{\text{old}}}](s), \ \forall s\in\mathcal{S}.
\end{align}
Then, $\pi_{\text{new}}$ is better than $\pi_{\text{old}}$, so that for every state $s$,
\begin{align}
\vspace{-5pt}
    V_{\pi_{\text{new}}}(s) \geq V_{\pi_{\text{old}}}(s). \nonumber
    \vspace{-5pt}
\end{align}
\end{restatable}
Of course, the monotonic improvement property of the expected return is a natural corollary of this lemma, as 
\begin{align}
    \eta(\pi_{\text{new}}) &= \E_{\rs\sim d}\big[V_{\pi_{\text{new}}}(\rs)] \nonumber\\
    &\geq 
    \E_{\rs\sim d}\big[V_{\pi_{\text{old}}}(\rs)] = \eta(\pi_{\text{old}}). \nonumber
\end{align}
Hence, optimisation of the mirror operator guarantees the improvement of the new policy's performance.
Condition (\ref{ineq:what-mirror-step-does}), however, is represented by $|\mathcal{S}|$ inequalities and solving them may be intractable in large state spaces. Hence, we shall design a proxy objective whose solution simply satisfies Condition (\ref{ineq:what-mirror-step-does}), and admits Monte-Carlo estimation (for practical reasons). For this purpose, we define the update rule of mirror learning.
\begin{restatable}{definition}{mirrorupdate}Let $\pi_{\text{old}}$ be a policy. Then, \textbf{mirror learning} updates the policy by 
\begin{align}
    \label{eq:mirror-update}
    \pi_{\text{new}}
    = \argmax_{\bar{\pi}\in\mathcal{N}(\pi_{\text{old}})}\E_{\rs\sim\beta_{\pi}}\Big[ \big[\mathcal{M}^{\bar{\pi}}_{\mathfrak{D}}V_{\pi_{\text{old}}}\big](\rs)\Big]. 
\end{align}
\end{restatable}
This, much simpler and clearer, problem formulation, is also robust to settings with large state spaces and parameterised policies. Being an expectation, it can be approximated by sampling with the unbiased ``batch'' estimator
\begin{align}
    \label{eq:monte-carlo}
    \hspace{-5pt}
    \frac{1}{|\mathfrak{B}|}\sum\limits_{\rs, \ra\in\mathfrak{B}}\Big[ \frac{\bar{\pi}(\ra|\rs)}{\pi_{\text{old}}(\ra|\rs)}Q_{\pi_{\text{old}}}(\rs, \ra) - \frac{\nu^{\bar{\pi}}_{\pi_{\text{old}}} (\rs)}{\beta_{\pi_{\text{old}}}(\rs)}\mathfrak{D}_{\pi_{\text{old}}}(\bar{\pi}|\rs)\Big]. 
\end{align}
The distribution $\nu^{\bar{\pi}}_{\pi_{\text{old}}}$ can be chosen equal to $\beta_{\pi_{\text{old}}}$, in which case the fraction from the front of the drift disappears, simplifying the estimator. Note also that one may choose $\beta_{\pi_{\text{old}}}$ independently of $\pi_{\text{old}}$. For example, it can be some fixed distribution over $\mathcal{S}$, like the uniform distribution. Hence, the above update supports \textbf{both on-policy and off-policy} learning. In the latter case, the fraction $\bar{\pi}(\ra|\rs)/\pi_{\text{old}}(\ra|\rs)$ is replaced by $\bar{\pi}(\ra|\rs)/\pi_{\text{hist}}(\ra|\rs)$, where $\pi_{\text{hist}}$ is the policy used to sample the pair $(\rs, \ra)$ and insert it to a replay buffer (for more details see Appendix \ref{appendix:off-policy}). Such a Monte-Carlo objective can be optimised with respect to $\bar{\pi}$ parameters by, for example, a few steps of gradient ascent. Most importantly, in its exact form, the update in Equation (\ref{eq:mirror-update}) guarantees that the resulting policy satisfies the desired Condition (\ref{ineq:what-mirror-step-does}).
\begin{restatable}{lemma}{mirrormajor}
\label{lemma:mirror-major}
Let $\pi_{\text{old}}$ be a policy and $\pi_{\text{new}}$ be obtained from the mirror update of Equation (\ref{eq:mirror-update}). Then, 
\begin{align}
         [\mathcal{M}^{\pi_{\text{new}}}_{\mathfrak{D}} V_{\pi_{\text{old}}}](s) \geq 
    [\mathcal{M}^{\pi_{\text{old}}}_{\mathfrak{D}} V_{\pi_{\text{old}}}](s), \ \forall s\in\mathcal{S}.
\nonumber
\end{align}
Hence, $\pi_{\text{new}}$ attains the properties provided by Lemma \ref{lemma:mpi}.
\end{restatable}

For proof, see Appendix \ref{appendix:lemmas}. Next, we use the above results to characterise the properties of mirror learning. The improvement and convegence properties that it exhibits not only motivate its usage, but may also contribute to explaining the empirical success of various widely-used algorithms. We provide a proof sketch below, and a detailed proof in Appendix \ref{appendix:theorem}.
\begin{restatable}[The Fundamental Theorem of Mirror Learning]{theorem}{mirrorfundamental} 
\label{theorem:fundamental}
Let $\mathfrak{D}^{\nu}$ be a drift, $\mathcal{N}$ be a neighbourhood operator, and the sampling distribution $\beta_{\pi}$ depend continuously on $\pi$. Let $\pi_0 \in \Pi$, and the sequence of policies $(\pi_n)_{n=0}^{\infty}$ be obtained by mirror learning induced by $\mathfrak{D}^{\nu}$, $\mathcal{N}$, and $\beta_{\pi}$. Then, the learned policies 
\begin{enumerate}[leftmargin=*]
    \vspace{-10pt}
    \item \label{property1} Attain the strict monotonic improvement property, 
    \begin{align}
    \vspace{-5pt}
        \eta(\pi_{n+1}) \ \geq \ \eta(\pi_n) + \E_{\rs\sim d }\Bigg[ \frac{\nu_{\pi_n}^{\pi_{n+1}}(\rs)}{\beta_{\pi_n}(\rs)}\mathfrak{D}_{\pi_n}(\pi_{n+1}|\rs)\Bigg],  \nonumber
    \vspace{-5pt}
    \end{align}
    \item \label{property2} Their value functions converge to the optimal one,
    \begin{align}
    \vspace{-5pt}
        \lim_{n\rightarrow\infty}V_{\pi_n} = V^*,\nonumber
    \vspace{-5pt}
    \end{align}
    \item \label{property3} Their expected returns converge to the optimal return,
    \begin{align}
    \vspace{-5pt}
        \lim_{n\rightarrow\infty}\eta(\pi_n) = \eta^*,\nonumber
    \vspace{-5pt}
    \end{align}
    \item \label{property4} Their $\omega$-limit set consists of the optimal policies.
\end{enumerate}
\end{restatable}
\begin{proofsketch}
We split the proof into four steps. In \textit{Step 1}, we use Lemmas \ref{lemma:mpi} \& \ref{lemma:mirror-major} to show that the sequence of value functions $(V_{\pi_k})_{k\in\mathbb{N}}$ converges. In \textit{Step 2}, we show the existance of limit points $\bar{\pi}$ of $(\pi_{k})_{k\in\mathbb{N}}$, and show that they are fixed points of the mirror learning update, which we do by contradiction. The most challenging \textit{Step 3}, is where we prove (by contradiction) that $\bar{\pi}$ is a fixed point of GPI. In the proof, we supposed that $\bar{\pi}$ is not a fixed point of GPI, and use the drift's zero-gradient property (Definition \ref{def:drift}) to show that one could slightly perturb $\bar{\pi}$ to obtain a policy $\pi'$, which corresponds to a higher mirror learning objective. This contradicts $\bar{\pi}$ simultaneously being a fixed point of mirror learning. \textit{Step 4} finalises the proof, by recalling that fixed points of GPI are optimal policies.
\end{proofsketch}
Hence, any algorithm whose update takes the form of Equation (\ref{eq:mirror-update}) improves the return monotonically, and converges to the set of optimal policies. This result provides RL algorithm designers with a template. New instances of it can be obtained by altering the drift $\mathfrak{D}^{\nu}$, the neighbourhood operator $\mathcal{N}$, and the sampling distribution function $\beta_{\pi}$. Indeed, in the next section, we show how some of the most well-known RL algorithms fit this template.
\vspace{-5pt}
\section{Mirror Learning View of RL Phenomena}
We provide a list of RL algorithms in their mirror learning representation in Appendix \ref{appendix:listing}.
\vspace{-5pt}
\paragraph{Generalised Policy Iteration} For the trivial drift $\mathfrak{D}^{\nu}\equiv 0$, and the trivial neighbourhood operator $\mathcal{N} \equiv \Pi$, the mirror learning update from Equation (\ref{eq:mirror-update}) is equivalent to
\begin{align}
\vspace{-5pt}
    \pi_{\text{new}} = \argmax_{\bar{\pi}\in\Pi} \E_{\rs\sim\beta_{\pi_{\text{old}}}}\big[ \E_{\ra\sim\bar{\pi}}\big[ Q_{\pi_{\text{old}}}(\rs, \ra) \big] \big].\nonumber
\vspace{-5pt}
\end{align}
As in this case the maximisation is unconstrained, and the expectation over $\rs\sim\beta_{\pi_{\text{old}}}$ is monotonically increasing in the individual conditional $\E_{\ra\sim\bar{\pi}}\big[ Q_{\pi_{\text{old}}}(s, \ra) \big]$, the maximisation distributes across states
\begin{align}
\vspace{-5pt}
    \pi_{\text{new}}(\cdot|s) = \argmax_{\bar{\pi}(\cdot|s)\in\mathcal{P}(\mathcal{A})} \E_{\ra\sim\bar{\pi}}\big[ Q_{\pi_{\text{old}}}(s, \ra) \big],\nonumber
\vspace{-5pt}
\end{align}
which is exactly the GPI update, as in Equation (\ref{eq:gpi}). Hence, all instances of GPI, \emph{e.g.,} policy iteration \citep{sutton2018reinforcement}, are special cases of mirror learning, and thus inherit its improvement and convergence properties. Furthermore, neural implementations, like A3C \citep{mnih2016asynchronous}, maintain these qualities approximately.

\paragraph{Trust-Region Learning} Let us choose the drift operator $\mathfrak{D}$ to be the scaled KL-divergence, so that $\mathfrak{D}_{\pi}(\bar{\pi}|s) = C_{\pi}\text{D}_{\text{KL}}\big( \pi(\cdot|s), \bar{\pi}(\cdot|s) \big)$, where $C_{\pi}$ is a constant that depends on $\pi$. Further, in the construction of the drift $\mathfrak{D}^{\nu}_{\pi}(\bar{\pi})$, let us set $\nu_{\pi}^{\bar{\pi}}(\rs) = \delta(\rs- s_{\max})$, where $\delta$ is the Dirac-delta distribution, and $s_{\max}$ is the state at which the KL-divergence between $\pi$ and $\bar{\pi}$ is largest. For the neighbourhood operator, we choose the trivial one. Lastly, for the sampling distribution, we choose $\beta_{\pi} = \bar{\rho}_{\pi}$, the normalised version of $\rho_{\pi}$. Then, a mirror learning step maximises
\begin{align}
\vspace{-5pt}
    \E_{\rs\sim\bar{\rho}_{\pi_{\text{old}}}}\Big[ \E_{\ra\sim\pi_{\text{new}}}\big[ A_{\pi_{\text{old}}}(\rs, \ra) \big] \Big] - C_{\pi_{\text{old}}}\text{D}_{\text{KL}}^{\text{max}}(\pi_{\text{old}}, \pi_{\text{new}}),\nonumber
\vspace{-5pt}
\end{align}
which, for appropriately chosen $C_{\pi_{\text{old}}}$, is proportional (and thus equivalent) to the trust-region learning update from Equation (\ref{eq:trust-region}). Hence, the monotonic improvement of trust-region learning follows from Theorem \ref{theorem:fundamental}, which also implies its convergence to the set of optimal policies.

\paragraph{TRPO} The algorithm is designed to approximate trust-region learning, and so its mirror-learning representation is similar. We make the following changes: set the drift operator to $\mathfrak{D}^{\nu}\equiv 0$, and choose the neighbourhood operator to the \textit{average-KL ball}. Precisely, 
\begin{align}
    \mathcal{N}(\pi) = \big\{ \bar{\pi} \in \Pi \ | \ \E_{\rs\sim \bar{\rho}_{\pi}}\big[ \text{D}_{\text{KL}}\big( \pi(\cdot|\rs), \bar{\pi}(\cdot|\rs) \big)\big] \leq \delta \big\}.\nonumber
\end{align}
The resulting mirror learning update is the learning rule of the TRPO algorithm. As a result, even though TRPO was supposed to only approximate the monotonic trust-region learning update \citep{trpo}, this analysis shows that TRPO has monontonic convergence guarantees.

\paragraph{PPO} We analyse PPO through unfolding the clipping objective $L^{\text{clip}}$ (Equation (\ref{eq:ppo})). For a given $s\in\mathcal{S}$, it equals
\begin{align}
    \E_{\ra\sim\pi_{\text{old}}}\big[ \min\big(\rr(\bar{\pi})A_{\pi_{\text{old}}}(s, \ra), \text{clip}\big( \rr(\bar{\pi}), 1\pm\epsilon\big)A_{\pi_{\text{old}}}(s, \ra)\big) \big].\nonumber
\end{align}
By recalling that $\rr(\bar{\pi}) = \bar{\pi}(\ra|s)/\pi_{\text{old}}(\ra|s)$, we use importance sampling \citep{sutton2018reinforcement}, and the trick of ``adding and subtracting'', to rewrite $L^{\text{clip}}$ as
\begin{align}
\vspace{-5pt}
    &\E_{\ra\sim\bar{\pi}}\Big[ A_{\pi_{\text{old}}}(s, \ra) \Big] - \E_{\ra\sim \pi_{\text{old}}}\big[ \rr(\bar{\pi}) A_{\pi_{\text{old}}}(s, \ra) \nonumber\\
    &- \min\big(\rr(\bar{\pi})A_{\pi_{\text{old}}}(s, \ra), \text{clip}\big( \rr(\bar{\pi}), 1\pm\epsilon\big)A_{\pi_{\text{old}}}(s, \ra)\big) \big].\nonumber
    \vspace{-5pt}
\end{align}
Going forward, we focus on the expectation that is subtracted. First, we can replace the $\min$ operator with $\max$, with the identity -$\min f(x) = \max \big[ -f(x) \big]$, as follows
\begin{align}
\vspace{-5pt}
    &\E_{\ra\sim \pi_{\text{old}}}\big[ \rr(\bar{\pi}) A_{\pi_{\text{old}}}(s, \ra) \nonumber\\
    &+ \max\big(-\rr(\bar{\pi})A_{\pi_{\text{old}}}(s, \ra), -\text{clip}\big( \rr(\bar{\pi}), 1\pm\epsilon\big)A_{\pi_{\text{old}}}(s, \ra)\big) \big].\nonumber
\vspace{-5pt}
\end{align}
Then, we move $\rr(\bar{\pi})A_{\pi_{\text{old}}}(s, \ra)$ inside the $\max$, and obtain
\begin{align}
\vspace{-5pt}
    &\E_{\ra\sim \pi_{\text{old}}}\big[ \max\big(0, \big[\rr(\bar{\pi})-\text{clip}\big( \rr(\bar{\pi}), 1\pm\epsilon\big)\big]A_{\pi_{\text{old}}}(s, \ra)\big) \big],\nonumber
\vspace{-5pt}
\end{align}
which can be simplified as
\begin{align}
    \label{eq:ppo-drift}
    \hspace{-6pt}
    \E_{\ra\sim \pi_{\text{old}}}\Big[ \text{ReLU}\Big( \big[\rr(\bar{\pi})-\text{clip}\big( \rr(\bar{\pi}), 1\hspace{-1pt}\pm\hspace{-1pt}\epsilon\big)\big]A_{\pi_{\text{old}}}(s, \ra)\Big) \Big] \
    \hspace{-4pt}
\end{align}
Notice now that this expression is always non-negative, due to the presence of the $\text{ReLU}$ function \citep{fukushima1982neocognitron}. Furthermore, for $\bar{\pi}$ sufficently close to $\pi_{\text{old}}$, \emph{i.e.,} so that $\rr(\bar{\pi})\in[1-\epsilon, 1+\epsilon]$, the clip operator reduces to the identity function, and so Equation (\ref{eq:ppo-drift}) is constant and zero. Hence, it is zero at $\bar{\pi}=\pi_{\text{old}}$, and has a zero gradient. Therefore, the expression in Equation (\ref{eq:ppo-drift}) is a drift functional of $\bar{\pi}(\cdot|s)$. This, together with taking $\beta_{\pi_{\text{old}}} \equiv \nu_{\pi_{\text{old}}}^{\bar{\pi}} \equiv \bar{\rho}_{\pi_{\text{old}}}$, and the trivial neighbourhood operator $\mathcal{N}\equiv \Pi$, shows that PPO, while it was supposed to be a heuristic approximation of trust-region learning, is by itself a rigorous instance of mirror learning. Thus, it inherits the monotonic improvement and convergence properties, which helps explain its great performance.

Interestingly, during the developmnent of PPO, a variant more closely related to TRL  was considered \citep{ppo}. Known as PPO-KL, the algorithm updates the policy to $\bar{\pi}=\pi_{\text{new}}$ to maximise
\begin{align}
\vspace{-5pt}
    \E_{\rs\sim\rho_{\pi_{\text{old}}}, \ra\sim\bar{\pi}}\big[ A_{\pi_{\text{old}}}(\rs, \ra)\big]- \tau\overline{\text{D}}_{\text{KL}}(\pi_{\text{old}}, \bar{\pi}),\nonumber
\vspace{-5pt}
\end{align}
and then scales $\tau$ up or down by a constant (typically $1.5$), depending on whether the KL-divergence induced by the update exceeded or subceeded some target level. Intriguingly, according to the authors this approach, although more closely-related to TRL, failed in practice. Mirror learning explains this apparent paradox. Namely, the rescaling scheme of $\tau$ causes discontinuity of the penalty term, when viewed as a function of $\pi_{\text{old}}$ and $\bar{\pi}$. This prevents the penalty from being a valid drift, so unlike PPO this version, PPO-KL, is not an instance of mirror learning. Recently, \citet{hsu2020revisiting} introduced PPO-KL with a fixed $\tau$, allowing for continuity of the KL-penalty. Such an algorithm is an instance of mirror learning, independently of whether the \textit{forward} or \textit{backward} KL-divergence is employed.
And indeed, the authors find that both versions of the algorithms result in strong empirical performance, further validating our theory.
% As this can be deduced from the mirror learning theory, regardless of whether the \textit{forward} or \textit{backward} KL-divergence was used, this modification resulted in a great empirical performance. 
\section{Related Work}
The RL community has long worked on the development of theory to aid the development of theoretically-sound techniques. Perhaps, one of the greatest achievements along this thread is the development of trust-region learning \citep[TRL]{trpo}---a framework that allows for stable training of monotonically improving policies. Unfortunately, despite its great theoeretical guarantees, TRL is intractable in most practical settings. Hence, the RL community focused on the development of heuristics, like TRPO \citep{trpo} and PPO \citep{ppo}, that approximate it while trading off theoretical guarantees of TRL for practicality. As these methods established new state-of-the-art performance on a variety of tasks \citep{duan2016benchmarking, berner2019dota}, the conceptual connection to TRL has been considered the key to success for many algorithms \citep{arjona2018rudder, queeney2021generalized}. However, recent works have shown that PPO is prone to breaking the core TRL principles of constraining the update size \citep{wang2020truly, engstrom2020implementation}, and thus revealed an existing chasm between RL practice and understanding---a problem that mirror learning resolves.
%(see Appendix \ref{sec:further_rel} for more details). 
%Notably, \citet{trpo} developed trust-region learning---a framework allowing for stable training of monotonically improving policies. However, a state-of-the-art heuristic algorithm derived from it, TRPO, largely deviates from the theory, and its theoretical properties are not fully understood. Nevertheless, it was the computational and implementational overwhelm of TRPO that concerned the community more and motivated it to develop alternative, efficient, TRL-inspired heuristics. As a result, \citet{ppo} developed PPO---an algorithm which incentifies an agent to make updates similar to TRPO. Surprisingly, in addition to bringing a computational advantage over its precedessor, this algorithm established a new SOTA performance on the most popular benchmarks and beyond \citep{berner2019dota}. This series of events revealed an existing dichotomy between RL understanding and practice.

Trust-region algorithms, although original, inspire connections between RL and optimisation, where the method of \textit{mirror descent }has been studied \citep{nemirovskij1983problem}. This idea has recently been extended upon in \cite{tomar2020mirror}, where \textit{Mirror Descent Policy Optimization} (MDPO) was proposed---the algorithm optimises the policy with mirror ascent, where the role of the Bregman distance is played by the KL-divergence. The method has been shown to achieve great empirical performance (which is also implied by mirror learning). Notably, a new stream of research in regularised RL is arising, where the regulariser is subsumed into the Bregman distance. For example, \citet{neu2017unified} have shown that, in the average-reward setting, a variant of TRPO is an instance of mirror descent in a regularised MDP, and converges to the optimal policy. \citet{shani2020adaptive} have generalised TRPO to handle regularised problems, and derived convergence rates for their methods in the tabular case. \citet{lan2021policy} and \citet{zhan2021policy} have proposed mirror-descent algorithms that solve regularised MDPs with convex regularisers, and also provided their convergence properties. Here, we would like to highlight that mirror learning is not a method of solving regularised problems through mirror descent. Instead, it is a very general class of algorithms that solve the classical MDP. The term \textit{mirror}, however, is inspired by the intuition behind mirror descent, which solves the image of the original problem under the \textit{mirror map} \citep{nemirovskij1983problem}---similarly, we defined the mirror operator. 

Mirror learning is also related to \textit{Probability Functional Descent} \citep[PFD]{chu2019probability}. Although PFD is a general framework of probability distribution optimisation in machine learning problems, in the case of RL, it is an instance of mirror learning---the one recovering GPI. Lastly, the concepts of mirror descent and functional policy representation are connected in a concurrent work of \citet{vaswani2021functional}, who show how the technique of mirror descent implements the functional descent of parameterised policies. This approach, although powerful, fails to capture some algorithms, including PPO, which we prove to be an instance of mirror learning. The key to the generalisation power of our theory is its abstractness, as well as simplicity.
\begin{figure}
    \centering
    \includegraphics[width=60mm]{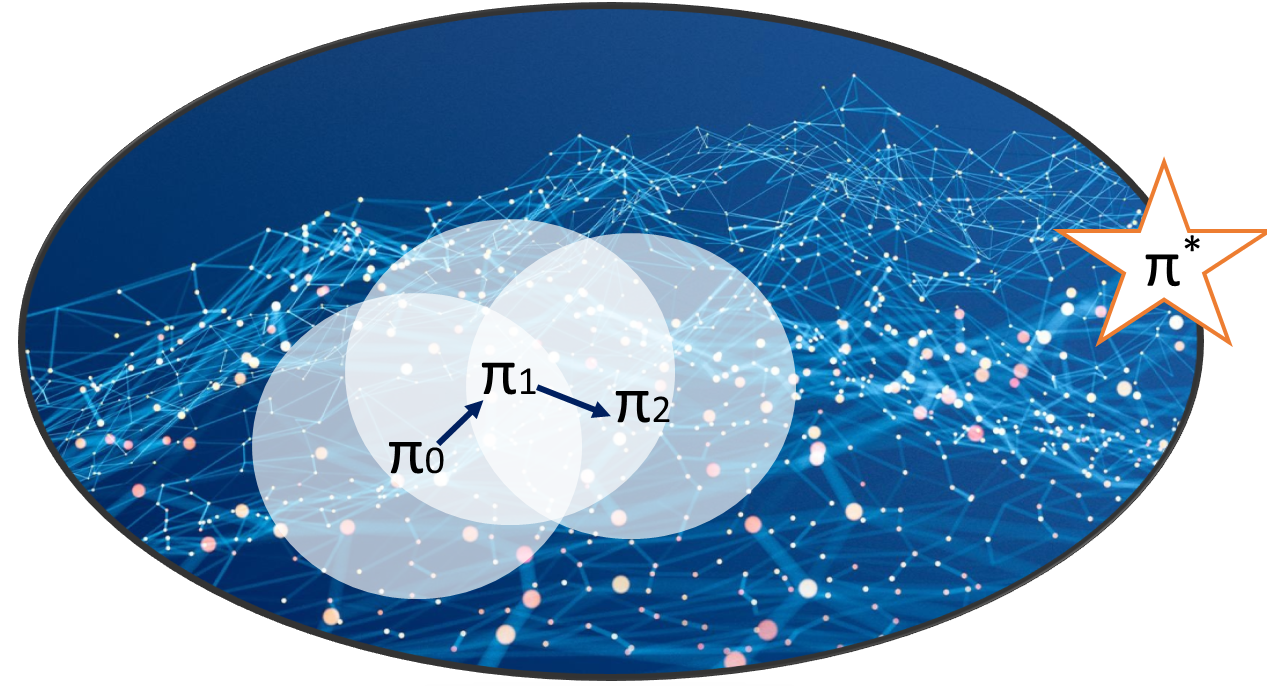}
    \caption{An intuitive view on the policy DAG and initial steps of mirror learning. A policy vertex has a neighbour, within its neighbourhood, which improves the return. }
    \label{fig:policy-dag}
    \vspace{-16pt}
\end{figure}
\vspace{-15pt}
\section{Graph-theoretical Interpretation}
In this section we use mirror learning to make a surprising connection between RL and graph theory. We begin by introducing the following definition of a particular \textit{directed acyclic graph} (DAG).
\begin{restatable}[Policy DAG]{definition}{policydag}
Let $\mathfrak{D}^{\nu}$ be a positive drift, and $\mathcal{N}$ be a neighbourhood operator. Then, the \textit{policy DAG} $\mathcal{G}(\Pi, \mathfrak{D}^{\nu}, \mathcal{N})$ is a graph where 
\begin{itemize}
    \vspace{-5pt}
    \item The vertex set is the policy space $\Pi$,
    \item $(\pi_1, \pi_2)$ is an edge if $\eta(\pi_1) < \eta(\pi_2)$ and $\pi_2\in\mathcal{N}(\pi_1)$,
    \item The weight of an edge $(\pi_1, \pi_2)$ is $\mathfrak{D}^{\nu}_{\pi_1}(\pi_2)$.
    \vspace{-5pt}
\end{itemize}
\end{restatable}
This graph is a valid DAG because the transitive, asymmetric ``$<$" relation that builds edges prevents the graph from having cycles. We also know that, for every non-optimal policy $\pi$, there in an outgoing edge $(\pi, \pi')$ from $\pi$, as $\pi'$ can be computed with a step of mirror learning (see Figure \ref{fig:policy-dag} for an intuitive picture).

The above definition allows us to cast mirror learning as a graph search problem. Namely, let $\pi_0 \in \Pi$ be a vertex policy at which we initialise the search, which further induces a sequence $(\pi_n)_{n=0}^{\infty}$. Let us define $\mathcal{U}_{\beta}=\min_{\pi\in\Pi, s\in\mathcal{S}}d(s)/\beta(s)$. Theorem \ref{theorem:fundamental} lets us upper-bound the weight of any traversed edge as
\begin{align}
    \vspace{-5pt}
    \label{ineq:strict-improvement}
    \eta(\pi_{n+1}) - \eta(\pi_n) \geq \mathcal{U}_{\beta}\cdot \mathfrak{D}^{\nu}_{\pi_n}(\pi_{n+1}).
    \vspace{-5pt}
\end{align}
Notice that $\eta(\pi_{n}) - \eta(\pi_0) = \sum_{i=1}^{n}\big[\eta(\pi_i) - \eta(\pi_{i-1})\big]$ converges to $\eta^* - \eta(\pi_0)$. Hence, the following series expansion 
\vspace{-12pt}
\begin{align}
\vspace{-5pt}
    \eta^* - \eta(\pi_0) = \sum\limits_{n=0}^{\infty}\big[ \eta(\pi_{n+1}) - \eta(\pi_n) \big]\nonumber
\vspace{-5pt}
\end{align}
is valid.
Combining it with Inequality (\ref{ineq:strict-improvement}), we obtain a bound on the total weight of the path induced by the search
\begin{align}   
\vspace{-5pt}
    \label{ineq:expected-mirror-bound}
    \frac{\eta^* - \eta(\pi_0)}{\mathcal{U}_{\beta}} \geq \sum\limits_{n=0}^{\infty}\mathfrak{D}^{\nu}_{\pi_n}(\pi_{n+1}),
\vspace{-5pt}
\end{align}
Hence, mirror learning finds a path from $\pi_0$ to the set of graph sinks (the optimal policies), whose total weight is finite, and does not depend on the policies that appeared during the search. Note that one could also estimate the left-hand side from above by a bound $(\eta^* +V_{\text{max}})/\mathcal{U}_{\beta}$ which is uniform for all initial policices $\pi_0$. This finding may be considered counterintuitive: While we can be decreasing the number of edges in the graph by shrinking the neighbourhood operator $\mathcal{N}(\pi)$, a path to $\pi^*$ still exists, and despite (perhaps) containing more edges, its total weight remains bounded.

Lastly, we believe that this inequality can be exploited for practical purposes: the drift functional $\mathfrak{D}$ is an abstract hyper-parameter, and the choice of it is a part of algorithm design. Practicioners can choose $\mathfrak{D}$ to describe a cost that they want to limit throughout training. For example, one can set $\mathfrak{D} = \texttt{risk}$, where $\texttt{risk}_{\pi}(\bar{\pi}|s)$ quantifies some notion of risk of updating $\pi(\cdot|s)$ to $\bar{\pi}(\cdot|s)$, or $\mathfrak{D}=\texttt{memory}$, to only make updates at a reasonable memory expense. Inequality (\ref{ineq:expected-mirror-bound}) guarantees that the total expense will remain finite, and provides an upper bound for it. Thereby, we encourage employing mirror learning with drifts designed to satisfy constraints of interest.
\vspace{-5pt}
\section{Numerical Experiments}
We verify the correctness of our theory with numerical experiments. Their purpose is not to establish a new state-of-the-art performance in the most challenging deep RL benchmarks. It is, instead, to demonstrate that algorithms that fall into the mirror learning framework obey Theorem \ref{theorem:fundamental} and Inequality (\ref{ineq:expected-mirror-bound}). Hence, to enable a close connection between the theory and experiments we choose simple environments and for drift functionals we selected: KL-divergence, squared L2 distance, squared total variation distance, and the trivial (zero) drift. For the neighbourhood operator, in one suite of experiments, we use the expected drift ball, \emph{i.e.}, the set of policies within some distance from the old policy, measured by the corresponding drift; in the second suite, we use the KL ball neighbourhood, like in TRPO. We represent the policies with $n$-dimensional action spaces as \textit{softmax} distributions, with $n-1$ parameters. For an exact verification of the theoretical results, we test each algorithm over \textit{only one} random seed. In all experiments, we set the initial-state and sampling distributions to uniform. The code is available at \url{https://github.com/znowu/mirror-learning}.

\textbf{Single-step Game.} In this game, the agent chooses one of $5$ actions, and receives a reward corresponding to it. These rewards are $10, 0, 1, 0, 5$, respectively for each action. The optimal return in this game equals $10$.

\textbf{Tabular Game.} This game has $5$ states, lined up next to each other. At each state, an agent can choose to go left, and receive a reward of $+0.1$, stay at the current state and receive $0$ reward, or go right and receive $-0.1$ reward. However, if the agent goes left at the left-most state, it receives $-10$ reward, and if it goes right at the right-most state, it recieves $+10$ reward. In these two cases, the game terminates. We set $\gamma=0.999$. Clearly, the optimal policy here ignores the small intermediate rewards and always chooses to go right. The corresponding optimal expected return is approximately $9.7$.

\textbf{GridWorld.} We consider a $5\times 5$ grid, with a barrier that limits the agent's movement. The goal of the agent is to reach the top-right corner as quick as possible, which terminates the episode, and to avoid the bottom-left corner with a bomb. It receives $-1$ reward for every step taken, and $-100$ for stepping onto the bomb and terminating the game.
For $\gamma=0.999$, the optimal expected return is approximately $-7.$
\begin{figure}[!ttbh]
    \centering
    \begin{subfigure}{0.49\linewidth}
        \centering
        \includegraphics[width=1.65in]{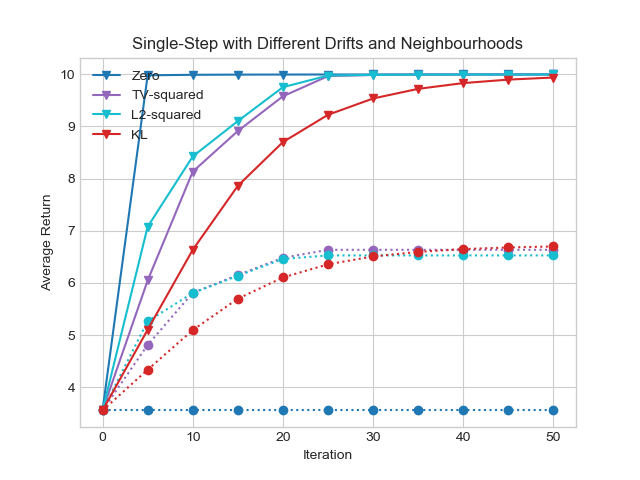}
    \end{subfigure}
    \begin{subfigure}{0.49\linewidth}
        \centering
        \includegraphics[width=1.65in]{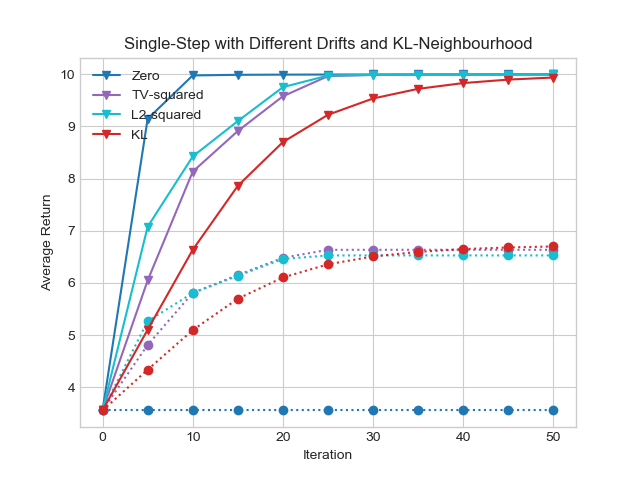} 
    \end{subfigure}
    \quad
    \begin{subfigure}{0.49\linewidth}
        \centering
        \includegraphics[width=1.65in]{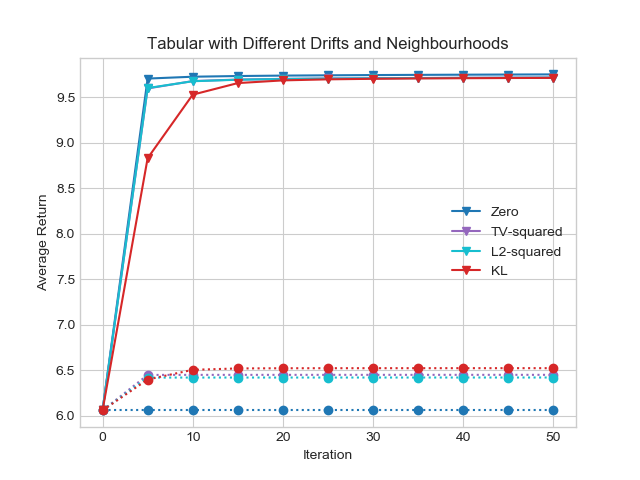}
    \end{subfigure}
    \begin{subfigure}{0.49\linewidth}
        \centering
        \includegraphics[width=1.65in]{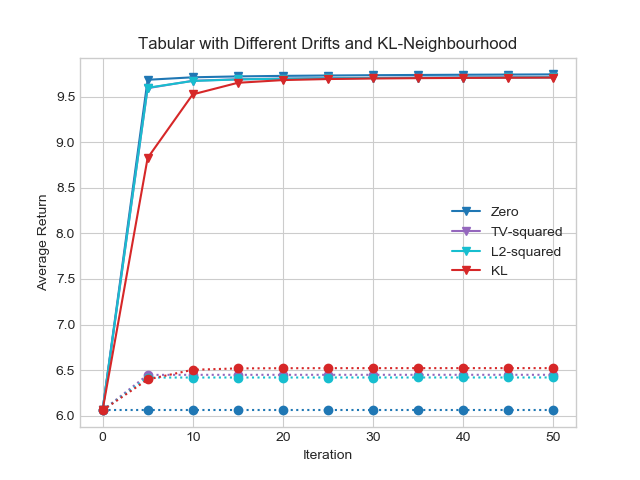}
    \end{subfigure}
    \quad
    \begin{subfigure}{0.49\linewidth}
        \centering
        \includegraphics[width=1.65in]{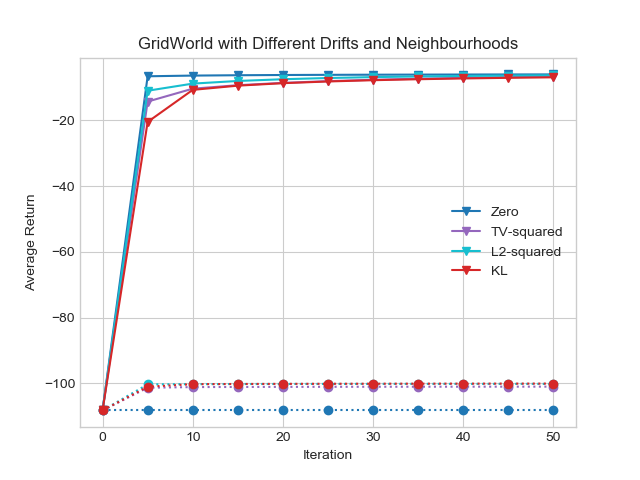}
    \end{subfigure}
    \begin{subfigure}{0.49\linewidth}
        \centering
        \includegraphics[width=1.65in]{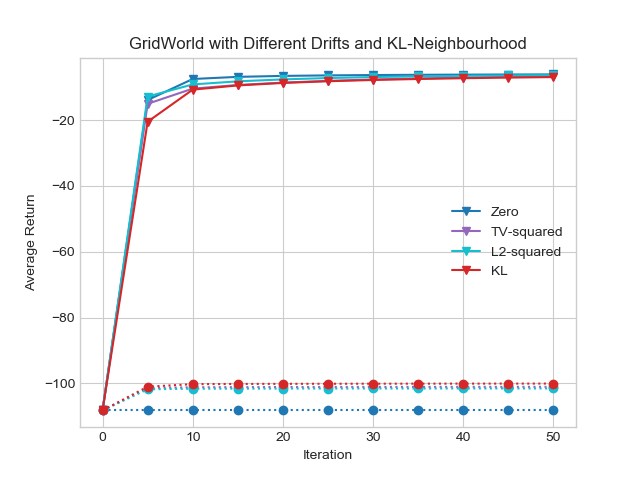}
    \end{subfigure}
    \caption{Mirror learning algorithms with different drifts and neighbourhood operators tested on simple environments. The solid lines represent the return, and the dotted ones represent the total drift. Algorithms in the left column use the drift ball neighbourhood, while those in the right use the KL ball. In both columns there are algorithms with each of the aforementioned drifts. Results are for one seed per environment and algorithm.}
    \label{fig:experiments}
    \vspace{-15pt}
\end{figure}

Figure \ref{fig:experiments} confirms that the resulting algorithms achieve the monotonic improvement property, and converges to the optimal returns---this confirms Theorem \ref{theorem:fundamental}. In these simple environments, the learning curves are influenced by the choice of the drift functional more than by the neighbourhood operator, although not significantly. An exception occurrs in the single-step game where, of course, one step of GPI is sufficient to solve the game. We also see that the value of the total drift (which we shift by $V_{\pi_{0}}$ on the plots for comparison) remains well below the return in all environments, confirming Inequality (\ref{ineq:expected-mirror-bound}). 

%Overall, the performance of different drifts and neighbourhoods, as well as their combinations, cannot be summarised within one work, as there are infinitely many of them. Hence, their design and properties open a new research question.
\vspace{-5pt}
\section{Conclusion}
In this paper, we introduced \textit{mirror learning}---a framework which unifies existing policy iteration algorithms. We have proved Theorem \ref{theorem:fundamental}, which states that any mirror learning algorithm solves the reinforcement learning problem. As a corollary to this theorem, we obtained convergence guarantees of state-of-the-art methods, including PPO and TRPO. More importantly, it provides a framework for the future development of theoretically-sound algorithms. We also proposed an interesting, graph-theoretical perspective on mirror learning, which establishes a connection between graph theory and RL. Lastly, we verifed the correctness of our theoretical results through numerical experiments on a diverse family of mirror learning instances in three simple toy settings. Designing and analysing the properties of the myriads of other possible mirror learning instances is an exciting avenue for future research. 

\section*{Acknowledgements}
I dedicate this work to my little brother \textit{Michał}, and thank him for persistently bringing light to my life. 
\newline
\textit{I have to go now. We won't see each other for a while, but one day we'll get together again. I promise.} \newline
\textit{Kuba}
\bibliography{example_paper.bib}
\bibliographystyle{icml2022}

\clearpage

\appendix
\onecolumn

\section{Definition Details}
\label{appendix:definitions}
\definedrift*

\textit{In this definition, the notion of the gradient of $\mathfrak{D}_{\pi}(\bar{\pi}|s)$ with respect to $\bar{\pi}(\cdot|s)\in\mathcal{P}(\mathcal{A})$ is rather intuitive. As the set $\mathcal{P}(\mathcal{A})$ is not a subset of $\mathbb{R}^n$, the gradient is not necessarily defined---however, we do not require its existence. Instead, as $\mathcal{P}(\mathcal{A})$ is a convex statistical manifold, we consider its G{\^ a}teaux derivatives \citep{gateaux1922diverses, hamilton1982global}. That is, for any $p\in\mathcal{P}(\mathcal{A})$, writing $v=p-\pi(\cdot|s)$, we consider the G{\^ a}teaux derivative, given as the limit
\begin{align}
    \delta \mathfrak{D}_{\pi}(\bar{\pi}|s) [v, \pi(\cdot|s)] \triangleq \lim_{\epsilon \rightarrow 0}\frac{\mathfrak{D}_{\pi}(\pi+\epsilon\cdot v|s)-\mathfrak{D}_{\pi}(\pi|s)}{\epsilon}
\end{align} 
and require them to be zero at $\bar{\pi}(\cdot|s)=\pi(\cdot|s)$. That is, we require that $\delta \mathfrak{D}_{\pi}(\bar{\pi}|s) [v, \pi(\cdot|s)]=0$.}

\bigskip

\neighbourhood*
\textit{We specify that a metric $\chi:\Pi\times\Pi\rightarrow\mathbb{R}$ is a metric between elements of a product of statistical manifolds, \emph{i.e.,} $\Pi = \bigtimes_{s\in\mathcal{S}}\mathcal{P}(\mathcal{A})$. Therefore, we carefully require that it is a non-negative, monotonically non-decreasing function of individual statistical divergence metrics $\chi_{s}\big(\pi(\cdot|s), \bar{\pi}(\cdot|s)\big)$, $\forall s\in\mathcal{S}$, and that $\chi\big(\pi, \bar{\pi}(\cdot|s)\big)=0$ only if $\chi_{s}\big(\pi(\cdot|s), \bar{\pi}(\cdot|s)\big)=0,\forall s\in\mathcal{S}$. 
}
\clearpage

\section{Proofs of Lemmas}
\label{appendix:lemmas}
\mpi*

\begin{proof}
Let us, for brevity, write $V_{\text{new}} = V_{\pi_{\text{new}}}$,  $V_{\text{old}} = V_{\pi_{\text{old}}}$,  $\beta=\beta_{\pi_{\text{old}}}$, $\nu=\nu_{\pi_{\text{old}}}^{\pi_{\text{new}}}$, and $\E_{\pi}[\cdot] = \E_{\ra\sim\pi, \rs'\sim P}[\cdot]$. For every state $s\in\mathcal{S}$, we have
\begin{align}
    &V_{\text{new}}(s) - V_{\text{old}}(s) \nonumber\\
    &= \E_{\pi_{\text{new}}}\big[ r(s, \ra) + \gamma V_{\text{new}}(\rs')\big]
    - \E_{\pi_{\text{old}}}\big[ r(s, \ra) + \gamma V_{\text{old}}(\rs')\big] \nonumber\\
    &= \E_{\pi_{\text{new}}}\big[ r(s, \ra) + \gamma V_{\text{old}}(\rs')\big] - \frac{\nu(s)}{\beta(s)}\mathfrak{D}_{\pi_{\text{old}}}(\pi_{\text{new}}|s) \nonumber\\
    &\quad - \E_{\pi_{\text{old}}}\big[ r(s, \ra) + \gamma V_{\text{old}}(\rs')\big]\nonumber\\
    &\quad + \gamma\E_{\pi_{\text{new}}}\big[
    V_{\text{new}}(\rs') - V_{\text{old}}(\rs')
    \big] + \frac{\nu(s)}{\beta(s)}\mathfrak{D}_{\pi_{\text{old}}}(\pi_{\text{new}}|s) \nonumber\\
    &=\E_{\pi_{\text{new}}}\big[ Q_{\pi_{\text{old}}}(s, \ra) \big] - \frac{\nu(s)}{\beta(s)}\mathfrak{D}_{\pi_{\text{old}}}(\pi_{\text{new}}|s)\nonumber\\
    &\quad- \E_{\pi_{\text{old}}}\big[ Q_{\pi_{\text{old}}}(s, \ra) \big] \nonumber\\
    &\quad +  \gamma\E_{\pi_{\text{new}}}\big[
    V_{\text{new}}(\rs') - V_{\text{old}}(\rs') \big] + \frac{\nu(s)}{\beta(s)}\mathfrak{D}_{\pi_{\text{old}}}(\pi_{\text{new}}|s)\nonumber\\
    &= \big[ \mathcal{M}^{\pi_{\text{new}}}_{\mathfrak{D}}V_{\text{old}}\big](s)
    - \big[ \mathcal{M}^{\pi_{\text{old}}}_{\mathfrak{D}}V_{\text{old}}\big](s)
    \nonumber\\
    &\quad + \gamma\E_{\pi_{\text{new}}}\big[
    V_{\text{new}}(\rs') - V_{\text{old}}(\rs')
    \big] + \frac{\nu(s)}{\beta(s)}\mathfrak{D}_{\pi_{\text{old}}}(\pi_{\text{new}}|s) \nonumber\\
    &\geq \gamma\E_{\pi_{\text{new}}}\big[
    V_{\text{new}}(\rs') - V_{\text{old}}(\rs')
    \big] + \frac{\nu(s)}{\beta(s)}\mathfrak{D}_{\pi_{\text{old}}}(\pi_{\text{new}}|s), \nonumber\\
    &\text{by Inequality (\ref{ineq:what-mirror-step-does}). Taking infimum within the expectation,}\nonumber\\
    &V_{\text{new}}(s) - V_{\text{old}}(s) \nonumber\\
    &\geq \gamma \cdot \inf_{s'}\big[  V_{\text{new}}(\rs') - V_{\text{old}}(\rs')
    \big] + \frac{\nu(s)}{\beta(s)}\mathfrak{D}_{\pi_{\text{old}}}(\pi_{\text{new}}|s).\nonumber\\
    &\text{Taking infimum over }s\nonumber\\
    &\inf_{s}\big[  V_{\text{new}}(s) - V_{\text{old}}(s)
    \big] \nonumber\\
    &\geq \gamma\inf_{s'}\big[  V_{\text{new}}(s') - V_{\text{old}}(s')
    \big] + \inf_{s}\frac{\nu(s)}{\beta(s)}\mathfrak{D}_{\pi_{\text{old}}}(\pi_{\text{new}}|s).\nonumber\\
    &\text{We conclude that}\nonumber\\
    &\inf_{s}\big[  V_{\text{new}}(s) - V_{\text{old}}(s)
    \big] \geq \inf_{s}\frac{\nu(s)}{\beta(s)}\mathfrak{D}_{\pi_{\text{old}}}(\pi_{\text{new}}|s)/(1-\gamma),\nonumber\\
    &\text{which is non-negative, and concludes the proof.} \nonumber
\end{align}
\end{proof}

\mirrormajor*
\begin{proof}
We will prove the statement by contradiction. Suppose that 
\begin{align}
    \label{equation:suppose-max}
    \pi_{\text{new}} = \argmax_{\bar{\pi}\in\mathcal{N}(\pi_{\text{old}})} \E_{\rs\sim\beta_{\pi_{\text{old}}}}\Big[ \big[ \mathcal{M}^{\bar{\pi}}_{\mathfrak{D}}V_{\pi_{\text{old}}}\big](\rs) \Big],
\end{align}
and that there exists a state $s_0$, such that 
\begin{align}
    \label{ineq:state}
    \big[ \mathcal{M}_{\mathfrak{D}}^{\pi_{\text{new}}}V_{\pi_{\text{old}}} \big](s_0) < 
    \big[ \mathcal{M}_{\mathfrak{D}}^{\pi_{\text{old}}}V_{\pi_{\text{old}}} \big](s_0). 
\end{align}
Let us define a policy $\hat{\pi}$, so that $\hat{\pi}(\cdot|s_0) = \pi_{\text{old}}(\cdot|s_0)$, and $\hat{\pi}(\cdot|s) = \pi_{\text{new}}(\cdot|s)$ for $s\neq s_0$. As the distance between of $\hat{\pi}$ from $\pi_{\text{old}}$ is the same as of $\pi_{\text{new}}$ at $s\neq s_0$, and possibly smaller at $s_0$, we have that $\bar{\pi}$ is not further from $\pi_{\text{old}}$ than $\pi_{\text{new}}$. Hence, $\hat{\pi} \in \mathcal{N}(\pi_{\text{old}})$. Furthermore, Inequality (\ref{ineq:state}) implies
\begin{align}
    \big[ \mathcal{M}_{\mathfrak{D}}^{\hat{\pi}}V_{\pi_{\text{old}}} \big](s_0) >
    \big[ \mathcal{M}_{\mathfrak{D}}^{\pi_{\text{new}}}V_{\pi_{\text{old}}} \big](s_0).
    \nonumber
\end{align}
Hence, 
\begin{align}
    &\E_{\rs\sim\beta_{\pi_{\text{old}}}}\Big[ \big[ \mathcal{M}^{\hat{\pi}}_{\mathfrak{D}}V_{\pi_{\text{old}}}\big](s) \Big] 
    -
    \E_{\rs\sim\beta_{\pi_{\text{old}}}}\Big[ \big[ \mathcal{M}^{\pi_{\text{new}}}_{\mathfrak{D}}V_{\pi_{\text{old}}}\big](s) \Big]\nonumber\\
    &=\big(\E_{\rs\sim\beta_{\pi_{\text{old}}}, \ra\sim \hat{\pi}}\Big[ Q_{\pi_{\text{old}}}(\rs, \ra) \Big] 
    -\mathfrak{D}^{v}_{\pi_{\text{old}}}(\hat{\pi})\big)
    - \big(\E_{\rs\sim\beta_{\pi_{\text{old}}}, \ra\sim \pi_{\text{new}}}\Big[ Q_{\pi_{\text{old}}}(\rs, \ra) \Big] 
    -\mathfrak{D}^{v}_{\pi_{\text{old}}}(\pi_{\text{new}})\big) \nonumber\\
    &\geq\big(\E_{\rs\sim\beta_{\pi_{\text{old}}}, \ra\sim \hat{\pi}}\Big[ Q_{\pi_{\text{old}}}(\rs, \ra) \Big] 
    -\mathfrak{D}^{v}_{\pi_{\text{old}}}(\pi_{\text{new}})\big)
    - \big(\E_{\rs\sim\beta_{\pi_{\text{old}}}, \ra\sim \pi_{\text{new}}}\Big[ Q_{\pi_{\text{old}}}(\rs, \ra) \Big] 
    -\mathfrak{D}^{v}_{\pi_{\text{old}}}(\pi_{\text{new}})\big) \nonumber\\
    &\text{because of monotonic non-decreasing of the drift in TV-distance}\nonumber\\
    &=\E_{\rs\sim\beta_{\pi_{\text{old}}}, \ra\sim \hat{\pi}}\Big[ Q_{\pi_{\text{old}}}(\rs, \ra) \Big] 
    - \E_{\rs\sim\beta_{\pi_{\text{old}}}, \ra\sim \pi_{\text{new}}}\Big[ Q_{\pi_{\text{old}}}(\rs, \ra) \Big] 
    \nonumber\\
    & = \beta_{\text{old}}(s_0) \Big( \big[ \mathcal{M}_{\mathfrak{D}}^{\hat{\pi}}V_{\pi_{\text{old}}} \big](s_0) -
    \big[ \mathcal{M}_{\mathfrak{D}}^{\pi_{\text{new}}}V_{\pi_{\text{old}}} \big](s_0) \Big) > 0.\nonumber
\end{align}
This is a contradiction with Equation (\ref{equation:suppose-max}), which finishes the proof. 
\end{proof}

\clearpage

\section{The Proof of Theorem \ref{theorem:fundamental}}
\label{appendix:theorem}

\mirrorfundamental*
\begin{proof}
We split the proof of the whole theorem into two parts, each of which proves different groups of properties stated in the theorem.
\newline
\textbf{Strict monotonic improvement  (Property \ref{property1})}
\newline 
By Lemma \ref{lemma:mpi}, we have that $\forall n\in\mathbb{N}, s\in\mathcal{S}$,
\begin{align}
    \label{ineq:weak-majorisation}
    V_{\pi_{n+1}}(s) \geq V_{\pi_n}(s).
    \vspace{-5pt}
\end{align}
Let us write $\beta=\beta_{\pi_n}$ and $\nu=\nu_{\pi_n}^{\pi_{n+1}}$. With this inequality, as well as Lemma \ref{lemma:mirror-major}, we have
\begin{align}
    &V_{\pi_{n+1}}(s) - V_{\pi_n}(s) \nonumber\\
    &= \E_{\ra\sim\pi_{n+1}, \rs' \sim P}\big[ r(s, \ra) + \gamma V_{\pi_{n+1}}(\rs') \big]\nonumber\\
    &\quad \quad \quad - \E_{\ra\sim\pi_{n}, \rs' \sim P}\big[ r(s, \ra) + \gamma V_{\pi_{n}}(\rs') \big]\nonumber\\
    &= \frac{\nu(s)}{\beta(s)}\mathfrak{D}_{\pi_n}(\pi_{n+1}|s) + \E_{\ra\sim\pi_{n+1}, \rs' \sim P}\big[ r(s, \ra) + \gamma V_{\pi_{n+1}}(\rs') \big]\nonumber\\
    &\quad- \frac{\nu(s)}{\beta(s)}\mathfrak{D}_{\pi_n}(\pi_{n+1}|s) 
    - \E_{\ra\sim\pi_{n}, \rs' \sim P}\big[ r(s, \ra) + \gamma V_{\pi_{n}}(\rs') \big].\nonumber\\
    &\text{Now, by Inequality (\ref{ineq:weak-majorisation})},\nonumber\\
    &\geq \frac{\nu(s)}{\beta(s)}\mathfrak{D}_{\pi_n}(\pi_{n+1}|s) + \E_{\ra\sim\pi_{n+1}}\big[ r(s, \ra) + \gamma V_{\pi_{n}}(\rs') \big] \nonumber\\
    &\quad - \frac{\nu(s)}{\beta(s)}\mathfrak{D}_{\pi_n}(\pi_{n+1}|s) - \E_{\ra\sim\pi_{n}}\big[ r(s, \ra) + \gamma V_{\pi_{n}}(\rs') \big]\nonumber\\
    &=\frac{\nu(s)}{\beta(s)}\mathfrak{D}_{\pi_n}(\pi_{n+1}|s) + \E_{\ra\sim\pi_{n+1}}\big[ Q_{\pi_{n}}(s, \ra) \big] \nonumber\\
    &\quad - \frac{\nu(s)}{\beta(s)}\mathfrak{D}_{\pi_n}(\pi_{n+1}|s) - \E_{\ra\sim\pi_{n}}\big[ Q_{\pi_{n}}(s, \ra) \big]\nonumber\\
    &= \frac{\nu(s)}{\beta(s)}\mathfrak{D}_{\pi_n}(\pi_{n+1}|s) + \big[ \mathcal{M}_{\mathfrak{D}}^{\pi_{n+1}}V_{\pi_n} \big](s) - \big[ \mathcal{M}_{\mathfrak{D}}^{\pi_{n}}V_{\pi_n} \big](s) \nonumber\\
    &\geq  \frac{\nu(s)}{\beta(s)}\mathfrak{D}_{\pi_n}(\pi_{n+1}|s),\nonumber
\end{align}
where the last inequality holds by Lemma \ref{lemma:mirror-major}.
To summarise, we proved that
\begin{align}
    \label{ineq:strict-majorisation}
    V_{\pi_{n+1}}(s) - V_{\pi_n}(s) \geq \frac{\nu(s)}{\beta(s)}\mathfrak{D}_{\pi_n}(\pi_{n+1}|s).
\end{align}
Taking the expectation over $\rs\sim d$, we obtain
\begin{align}
    \eta(\pi_{n+1}) - \eta(\pi_n) \geq \E_{\rs\sim d}\Bigg[ \frac{\nu(\rs)}{\beta(\rs)}\mathfrak{D}_{\pi_n}(\pi_{n+1}|\rs)\Bigg],\nonumber
\end{align}
which finishes this part of the proof.

\textbf{Optimality (Properties \ref{property2}, \ref{property3}, \& \ref{property4})}
\newline
\textbf{Step 1 (convergence of value functions).}
\newline
We consider the sequence $(V_{\pi_n})_{n=0}^{\infty}$ of associated value functions. By Lemma \ref{lemma:mirror-major}, we know that $\forall n\in\mathbb{N}$, $V_{\pi_{n+1}}\geq V_{\pi_n}$ uniformly. However, as for every $n\in\mathbb{N}$ and $s\in\mathcal{S}$, the uniform bound $V_{\pi_n}(s)\leq V_{\max}$ holds, the sequence of value functions must converge. We denote its limit by $V$.

\textbf{Step 2 (characterisation of policy limit points). } 
\newline
Let $L\Pi$ be the set of all limit points of $(\pi_n)_{n=0}^{\infty}$. As the sequence $(\pi_n)_{n=0}^{\infty}$ is bounded, Bolzano-Weierstrass guarantees that there exists $\bar{\pi}\in L\Pi$, and a subsequence, say $(\pi_{n_i})_{i=1}^{\infty}$, which converges to it. Writing $\beta_{k} = \beta_{\pi_k}$, we consider the optimisation problem that mirror learning induces at every element $\pi_{n_i}$ of this subsequence,
\begin{align}
    \label{eq:optimisation-of-subseq}
    \max_{\pi\in\mathcal{N}(\pi_{n_i})}\E_{\rs\sim\beta_{n_i}}\Big[ \big[\mathcal{M}^{\pi}_{\mathfrak{D}}V_{\pi_{n_i}} \big](\rs) \Big] 
\end{align}
Note that by the continuity of the value function \citep[Appendix A]{kuba2021trust}, as well as the drift, the neighbourhood operator, and the sampling distribution, we obtain that the above expression is continuous in $\pi_{n_i}$. Hence, by Berge's Maximum Theorem \citep{ausubel1993generalized},
writing $\bar{\beta} = \beta_{\bar{\pi}}$, as $i\rightarrow\infty$, the above expression converges to the following, 
\begin{align}  
    \label{eq:limit-of-max}
    \max_{\pi\in\mathcal{N}(\bar{\pi})}\E_{\rs\sim\bar{\beta}}\Big[ \big[\mathcal{M}^{\pi}_{\mathfrak{D}}V_{\bar{\pi}} \big](\rs) \Big].
\end{align}
\textit{Note that $V_{\bar{\pi}} = V$, as the sequence of value functions converges (has a unique limit point $V$).}
Furthermore, for all $i\in\mathbb{N}$, $\pi_{n_i + 1}$ is the argmax (precisely, is an element of the upper hemicontinuous argmax correspondence) of Equation (\ref{eq:optimisation-of-subseq}). Hence, there exists a subsequence $(\pi_{n_{i_k}+1})$ of $(\pi_{n_i+1})_{i\in\mathbb{N}}$ that converges to a policy $\pi'$, which is a solution to Equation (\ref{eq:limit-of-max}).
\newline
\textbf{Claim}: \textit{The solution to Equation} (\ref{eq:limit-of-max}) \textit{is} $\pi'=\bar{\pi}$. 
\newline We prove the above claim by contradiction. Suppose that $\pi'\neq \bar{\pi}$. As $\pi'$ can be obtained from $\bar{\pi}$ via mirror learning, Inequality (\ref{ineq:weak-majorisation}) implies
\begin{align}
    \label{ineq:q-major}
    Q_{\pi'}(s, a) &= \E_{\rs'\sim P}\big[ r(s, a) + \gamma V_{\pi'}(\rs') \big] \nonumber\\
    &\geq \E_{\rs'\sim P}\big[ r(s, a) + \gamma V_{\bar{\pi}}(\rs') \big] = Q_{\bar{\pi}}(s, a). 
\end{align}
If we have 
\begin{align}
    \E_{\rs\sim \bar{\beta}}\Big[ \big[ \mathcal{M}^{\pi'}_{\mathfrak{D}}V_{\bar{\pi}}\big](\rs) \Big] > 
    \E_{\rs\sim \bar{\beta}}\Big[ \big[ \mathcal{M}^{\bar{\pi}}_{\mathfrak{D}}V_{\bar{\pi}}\big](\rs) \Big],\nonumber
\end{align}
then for some state $s$, 
\begin{align}
    \big[ \mathcal{M}^{\pi'}_{\mathfrak{D}}V_{\bar{\pi}}\big](s) > \big[ \mathcal{M}^{\bar{\pi}}_{\mathfrak{D}}V_{\bar{\pi}}\big](s),\nonumber
\end{align}
which can be written as 
\begin{align}
    &\E_{\ra\sim\pi'}\big[ Q_{\bar{\pi}}(s, \ra) \big] - \frac{\nu^{\pi'}_{\bar{\pi}}(s)}{\bar{\beta}(s)}\mathfrak{D}_{\bar{\pi}}(\pi'|s) \nonumber\\
    &> 
    \E_{\ra\sim\bar{\pi}}\big[ Q_{\bar{\pi}}(s, \ra) \big] - \frac{\nu_{\bar{\pi}}^{\bar{\pi}}(s)}{\bar{\beta}(s)}\mathfrak{D}_{\bar{\pi}}(\bar{\pi}|s) = V_{\bar{\pi}}(s) = V(s). \nonumber
\end{align}
Using Inequality (\ref{ineq:q-major}), and non-negativity of drifts,
\begin{align}
V_{\pi'}(s) &= \E_{\ra\sim\pi'}\big[ Q_{\pi'}(s, \ra) \big]\nonumber\\
&\geq \E_{\ra\sim\pi'}\big[ Q_{\bar{\pi}}(s, \ra) \big] - \frac{\nu(s)}{\bar{\beta}(s)}\mathfrak{D}_{\bar{\pi}}(\pi'|s)  > V(s).\nonumber
\end{align}
However, as $V_{\pi'}$ is the limit of $V_{\pi_{n_{i_k}+1}}$, we have $V_{\pi'}=V$ (uniqueness of the value limit) which yields a contradiction, proving the claim.
Hence, the limit point $\bar{\pi}$ of the sequence $(\pi_n)_{n=0}^{\infty}$ satisfies
\begin{align}
    \label{eq:pi-bar-optimal-drift}
    \bar{\pi} = \argmax_{\pi\in\mathcal{N}(\bar{\pi})}\E_{\rs\sim\bar{\beta}}\Big[ \big[\mathcal{M}^{\pi}_{\mathfrak{D}}V_{\bar{\pi}} \big](\rs) \Big]. 
\end{align}
\textbf{Step 3 (dropping the drift).}
\newline
Let $\bar{\pi}$ be a limit point of $(\pi_n)_{n=0}^{\infty}$. From Equation (\ref{eq:pi-bar-optimal-drift}), and the definition of the mirror operator, we know that
\begin{align}
    \label{eq:optimality-rewritten}
    \bar{\pi} = \argmax_{\pi\in\mathcal{N}(\bar{\pi})}\E_{\rs\sim\bar{\beta}, \ra\sim\pi}\big[ A_{\bar{\pi}}(\rs, \ra) \big] - \mathfrak{D}^{\nu}_{\bar{\pi}}(\pi).
\end{align}
Suppose that there exists a policy $\pi'$, and a state $s$, such that 
\begin{align}
    \label{ineq:strictly-better-pi}
    \E_{\ra\sim\pi'}\big[ A_{\bar{\pi}}(s, \ra) \big] > \E_{\ra\sim\bar{\pi}}\big[ A_{\bar{\pi}}(s, \ra) \big] = 0.
\end{align}
For any policy $\pi$, consider the canonical parametrisation $\pi(\cdot|s) = (x_1, \dots, x_{m-1}, 1-\sum_{i=1}^{m-1}x_i)$, where $m$ is the size of the action space. We have that
\begin{align}
    &\E_{\ra\sim\pi}\big[ A_{\bar{\pi}}(s, \ra) \big] = \sum_{i=1}^{m}\pi(a_i|s)A_{\bar{\pi}}(s, a_i)\nonumber\\
    &= \sum_{i=1}^{m-1}x_i A_{\bar{\pi}}(s, a_i) 
    + (1-\sum_{j=1}^{m-1}x_j) A_{\bar{\pi}}(s, a_m)\nonumber\\
    &= \sum_{i=1}^{m-1}x_i \big[ A_{\bar{\pi}}(s, a_i) - A_{\bar{\pi}}(s, a_m) \big]
    +A_{\bar{\pi}}(s, a_m).\nonumber
\end{align}
This means that $\E_{\ra\sim\pi}\big[ A_{\bar{\pi}}(s, \ra) \big]$ is an affine function of $\pi(\cdot|s)$, and thus, its G{\^ a}teaux derivatives are constant in $\mathcal{P}(\mathcal{A})$ for fixed directions. Together with Inequality (\ref{ineq:strictly-better-pi}), this implies that the G{\^ a}teaux derivative, in the direction from $\bar{\pi}$ to $\pi'$, is strictly positive. Furthermore, the G{\^a}teaux derivatives of $\frac{\nu^{\pi}_{\bar{\pi}}(s)}{\bar{\beta}(s)}\mathfrak{D}_{\bar{\pi}}(\pi|s)$ are zero at $\pi(\cdot|s)=\bar{\pi}(\cdot|s)$, as
\begin{align}
    &0 = \frac{\nu^{\bar{\pi}}_{\bar{\pi}}(s)}{\bar{\beta}(s)}\mathfrak{D}_{\bar{\pi}}(\bar{\pi}|s) = \frac{1}{\bar{\beta}(s)}\mathfrak{D}_{\bar{\pi}}(\bar{\pi}|s),\nonumber\\
    & \quad \quad \quad \quad \quad \quad \quad \text{and}\nonumber\\
    &0 \leq \frac{\nu_{\bar{\pi}}^{\pi}(s)}{\bar{\beta}(s)}\mathfrak{D}_{\bar{\pi}}(\pi|s) \leq \frac{1}{\bar{\beta}(s)}\mathfrak{D}_{\bar{\pi}}(\pi|s),\nonumber
\end{align}
and both the lower and upper of the above bounds have zero derivatives.
Hence, the G{\^ a}teaux derivative of $\E_{\ra\sim\pi}\big[ A_{\bar{\pi}}(s, \ra) \big] - \frac{\nu^{\pi}_{\bar{\pi}}(s)}{\bar{\beta}(s)}\mathfrak{D}_{\bar{\pi}}(\pi|s)$ is strictly positive. Therefore, for conditional policies $\hat{\pi}(\cdot|s)$ sufficiently close to $\bar{\pi}(\cdot|s)$ in the direction towards $\pi'(\cdot|s)$, we have (with a slight abuse of notation for $\nu$)
\begin{align}
    \label{ineq:pi-hat}
    \E_{\ra\sim\hat{\pi}}\big[ A_{\bar{\pi}}(s, \ra) \big] - \frac{\nu_{\bar{\pi}}^{\hat{\pi}}(s)}{\bar{\beta}(s)}\mathfrak{D}_{\bar{\pi}}(\hat{\pi}|s) > 0.
\end{align}
Let us construct a policy $\widetilde{\pi}$ as follows. For all states $y\neq s$, we set $\widetilde{\pi}(\cdot|y) = \bar{\pi}(\cdot|y)$. Moreover, for $\widetilde{\pi}(\cdot|s)$ we choose $\hat{\pi}(\cdot|s)$ as in Inequality (\ref{ineq:pi-hat}), sufficiently close to $\bar{\pi}(\cdot|s)$, so that $\widetilde{\pi}\in\mathcal{N}(\bar{\pi})$. Then, we have
\begin{align}
    &\E_{\rs\sim\bar{\beta}, \ra\sim\widetilde{\pi}}\big[ A_{\bar{\pi}}(\rs, \ra) \big] - \mathfrak{D}^{\nu}_{\bar{\pi}}(\widetilde{\pi})\nonumber\\
    &= \bar{\beta}(s)\big[ \E_{\ra\sim\hat{\pi}}\big[ A_{\bar{\pi}}(s, \ra) \big] - \frac{\nu_{\bar{\pi}}^{\hat{\pi}}(s)}{\beta_{\bar{\pi}}(s)}\mathfrak{D}_{\bar{\pi}}(\hat{\pi}|s)  \big] > 0.
\end{align}
The above contradicts Equation (\ref{eq:optimality-rewritten}). Hence, the assumption made in Inequality (\ref{ineq:strictly-better-pi}) was false. Thus, we have proved that, for every state $s$,
\begin{align}
    \label{eq:optimality}
    \bar{\pi} &= \argmax_{\pi}\E_{\ra\sim\pi}\big[ A_{\bar{\pi}}(s, \ra) \big]\nonumber\\
    &= \argmax_{\pi}\E_{\ra\sim\pi}\big[ Q_{\bar{\pi}}(s, \ra) \big].
\end{align}
\newline
\textbf{Step 4 (optimality).}
\newline
Equation (\ref{eq:optimality}) implies that $\bar{\pi}$ is an optimal policy \citep{sutton2018reinforcement}, and so the value function $V=V_{\bar{\pi}}$ is the optimal value function $V^*$ (Property \ref{property2}). Consequently, the expected return that the policies converge to, $\eta = \E_{\rs\sim d}\big[ V(\rs) \big] = \E_{\rs\sim d}\big[ V^*(\rs) \big]=\eta^*$ is optimal (Property \ref{property3}).  Lastly, as $\bar{\pi}$ was an arbitrary limit point, any element of the $\omega$-limit set is an optimal policy (Property \ref{property4}). This finishes the proof.
\end{proof}

\section{Extension to Continuous State and Action Spaces}
The results on Mirror Learning extend to continuous state and action spaces through general versions of our claims. These require a little more care in their formulation, but their validity holds as a corollary to our proofs.

In general, the state and the action spaces $\mathcal{S}$ and $\mathcal{A}$ must be assumed to be compact and measurable. For the state space, we introduce a reference probability measure $\mu_{\mathcal{S}}:\mathbb{P}(\mathcal{S})\rightarrow\mathbb{R}$ that is strictly positive, \emph{i.e}, $\mu_{\mathcal{S}}(s)>0, \forall s\in\mathcal{S}$. Under such setting, a policy $\pi^*$ is optimal if it satisfies the Bellman optimality equation 
\begin{align}
    \pi^*(\cdot|s) = \argmax\limits_{\pi(\cdot|s)\in\mathcal{P}(\mathcal{A})}\E_{\ra\sim p}\big[ Q_{\pi^*}(s, \ra)\big],\nonumber
\end{align}
at states that form a set of measure $1$ with respect to $\mu_{\mathcal{S}}$. In other words, a policy is optimal if it obeys the Bellman optimality equation almost surely.

As for the results, the inequality provided by Lemma \ref{lemma:mpi} (the state value function improvement) holds almost surely with respect to $\mu_{\mathcal{S}}$ as long as the policy $\pi_{\text{new}}$ satisfies Inequality (\ref{ineq:what-mirror-step-does}), also almost surely with respect to this measure. Of course, the corollary on the monotonic improvement remains valid. Next, the inequality introduced in Lemma \ref{lemma:mirror-major} must now be stated almost surely, again with respect to $\mu_{\mathcal{S}}$. Lastly, the entire statement of Theorem \ref{theorem:fundamental} remains unchanged---it has the same formulation for all compact .

\clearpage
\section{The Listing of RL Approaches as Instances of Mirror Learning}
\label{appendix:listing}
\subsection*{Generalised Policy Iteration \citep{sutton2018reinforcement}}
\vspace{-10pt}
\begin{align}
    \label{eq:gpi}
    \pi_{\text{new}}(\cdot|s) = \argmax_{\bar{\pi}(\cdot|s)\in\mathcal{P}(\mathcal{A})} \E_{\ra\sim\bar{\pi}}\big[ Q_{\pi_{\text{old}}}(s, \ra)\big], \ \forall s\in\mathcal{S}\nonumber
\end{align}
\begin{itemize}
    \item Drift functional: trivial $\mathfrak{D} \equiv 0$.
    \item Neighbourhood operator: trivial $\mathcal{N} \equiv \Pi$.
    \item Sampling distribution: arbitrary.
\end{itemize}

\subsection*{Trust-Region Learning \citep{trpo}}
\vspace{-10pt}
\begin{align}
    &\pi_{\text{new}} = \argmax_{\bar{\pi}\in\Pi}\E_{\rs\sim\rho_{\pi_{\text{old}}}, \ra\sim\bar{\pi}}\big[ A_{\pi_{\text{old}}}(\rs, \ra) \big] - C\text{D}_{\text{KL}}^{\text{max}}(\pi_{\text{old}}, \bar{\pi}),\nonumber\\
    &\quad \quad \quad \quad \text{where }C = \frac{4\gamma \max_{s, a}|A_{\pi_{\text{old}}}(s, a)|}{(1-\gamma)^2}\nonumber
\end{align}
\begin{itemize}
    \item Drift functional: 
    \begin{align}
        &\mathfrak{D}_{\pi}(\bar{\pi}|s) = (1-\gamma)C\text{D}_{\text{KL}}\big( \pi(\cdot|s), \bar{\pi}(\cdot|s) \big),\ 
        \text{ with }\nu_{\pi}^{\bar{\pi}}(s_{\max}) = 1, \text{ where } s_{\max}=\argmax_{s\in\mathcal{S}}\mathfrak{D}_{\pi}(\bar{\pi}|s). \nonumber 
    \end{align}
    \item Neighbourhood operator: trivial $\mathcal{N} \equiv \Pi$.
    \item Sampling distribution: $\beta_{\pi} = \bar{\rho}_{\pi}$, the normalised marginal discounted state distribution.
\end{itemize}

\subsection*{TRPO \citep{trpo}}
\vspace{-10pt}
\begin{align}
    &\quad \quad \pi_{\text{new}} = \argmax_{\bar{\pi}\in\Pi}\E_{\rs\sim\rho_{\pi_{\text{old}}}, \ra\sim\bar{\pi}}\big[ A_{\pi_{\text{old}}}(\rs, \ra) \big] \nonumber\\
    &\text{subject to }\E_{\rs\sim\rho_{\pi_{\text{old}}}}\big[ \text{D}_{\text{KL}}\big(\pi_{\text{old}}(\cdot|\rs), \pi_{\text{new}}(\cdot|\rs) \big) \big] \leq \delta.\nonumber 
\end{align}
\begin{itemize}
    \item Drift functional: trivial $\mathfrak{D} \equiv 0$.
    \item Neighbourhood operator:
    \begin{align}
        \mathcal{N}(\pi) = \{ \bar{\pi} \in \Pi : \E_{\rs\sim\rho_{\pi}}\big[ D_{\text{KL}}\big( \pi(\cdot|\rs), \bar{\pi}(\cdot|\rs) \big)\big] \leq \delta \big\}.\nonumber
    \end{align}
    \item Sampling distribution: $\beta_{\pi} = \bar{\rho}_{\pi}$.
\end{itemize}

\subsection*{PPO \citep{ppo}}
\vspace{-10pt}
\begin{align}
    &\quad \quad \quad \quad \pi_{\text{new}} = \argmax_{\bar{\pi}\in\Pi} \E_{\rs\sim\bar{\rho}_{\pi_{\text{old}}}, \ra\sim\pi_{\text{old}}}\big[ L^{\text{clip}}\big], \nonumber\\
    &\quad \quad \quad \text{where for given }s, a,  \text{ and }  \rr(\bar{\pi}) = \frac{\bar{\pi}(a|s)}{\pi_{\text{old}}(a|s)}\nonumber\\
    &L^{\text{clip}}= \min\Big( \rr(\bar{\pi}) A_{\pi_{\text{old}}}(s, a), \text{clip}\big( \rr(\bar{\pi}), 1\pm\epsilon\big)A_{\pi_{\text{old}}}(s, a)\Big).\nonumber
\end{align}
\begin{itemize}
    \item Drift functional:
    \begin{align}
        &\mathfrak{D}_{\pi}(\bar{\pi}|s) = \E_{\ra\sim\pi}\Big[ \text{ReLU}\Big( \big[\rr(\bar{\pi}) - \text{clip}\big( \rr(\bar{\pi}), 1\pm \epsilon\big) \big] A_{\pi}(s, \ra) \Big) \Big], \text{ with }\nu_{\pi}^{\bar{\pi}} = \bar{\rho}_{\pi}.\nonumber
    \end{align}
    \item Neighbourhood operator: trivial $\mathcal{N}\equiv \Pi$. Note, however, that in practical implementations the policy is updated with $\epsilon_{\text{PPO}}$ steps of gradient ascent with gradient-clipping threshold $M$. This corresponds to a neighbourhood of an L2-ball of radius $M\epsilon_{\text{PPO}}$ in the policy parameter space.
    \item Sampling distribution: $\beta_{\pi} = \bar{\rho}_{\pi}$.
\end{itemize}

\subsection*{PPO-KL \citep{hsu2020revisiting}}
\vspace{-10pt}
\begin{align}
    \pi_{\text{new}} = \argmax_{\bar{\pi}\in\Pi}\E_{\rs\sim\bar{\rho}_{\pi_{\text{old}}}, \ra\sim\bar{\pi}}\big[ A_{\pi_{\text{old}}}(\rs, \ra)\big] - \tau\overline{\text{D}}_{\text{KL}}(\pi_{\text{old}}, \bar{\pi}),\nonumber
\end{align}

\begin{itemize}
    \item Drift functional:
    \begin{align}
        &\mathfrak{D}_{\pi}(\bar{\pi}|s) = \tau \text{D}_{\text{KL}}\big(\pi(\cdot|\rs), \bar{\pi}(\cdot|\rs)\big), \text{ with }\nu_{\pi}^{\bar{\pi}} = \bar{\rho}_{\pi}.\nonumber
    \end{align}
    \item Neighbourhood operator: the same as in PPO.
    \item Sampling distribution: $\beta_{\pi} = \bar{\rho}_{\pi}$.
\end{itemize}

\subsection*{MDPO \citep{tomar2020mirror}}
\vspace{-10pt}
\begin{align}
    \pi_{\text{new}} = \argmax_{\bar{\pi}\in\Pi}\E_{\rs\sim\beta_{\pi_{\text{old}}}, \ra\sim\bar{\pi}}\big[ A_{\pi_{\text{old}}}(\rs, \ra)\big] - \frac{1}{t_{\pi_{\text{old}}}}\overline{\text{D}}_{\text{KL}}\big(\bar{\pi}, \pi_{\text{old}} \big).\nonumber
\end{align}
\begin{itemize}
    \item Drift functional:
    \begin{align}
        &\mathfrak{D}_{\pi}(\bar{\pi}|s) = \frac{1}{t_{\pi}}D_{\text{KL}}\big( \bar{\pi}(\cdot|s), \pi(\cdot|s) \big), \quad \text{ with }\nu_{\pi}^{\bar{\pi}}= \beta_{\pi}.\nonumber
    \end{align}
    \item Neighbourhood operator: trivial $\mathcal{N} = \Pi$.
    \item Sampling distribution: $\beta_{\pi} = \bar{\rho}_{\pi}$ for on-policy MDPO, and $\beta_{\pi_{n}}(s) = \frac{1}{n+1}\sum_{i=0}^{n} \bar{\rho}_{\pi_i}(s)$ for off-policy MDPO.
\end{itemize}

\clearpage

\section{Instructions for Implementation of Off-Policy Mirror Learning}
\label{appendix:off-policy}
In the case of off-policy learning, estimating $\E_{\ra\sim\bar{\pi}}\big[ Q_{\pi_{\text{old}}}(s, \ra) \big]$ is not as straighforward as in Equation (\ref{eq:monte-carlo}), since sampling actions from the replay buffer is not equivalent to sampling actions from $\pi_{\text{old}}(\cdot|s)$ anymore. The reason for this is that while sampling an action from the buffer, we also sample a past policy, $\pi_{\text{hist}}$, which was used to insert the action to the buffer. To formalise it, we draw a past policy from some distribution dictated by the buffer, $\pi_{\text{hist}}\sim h\in\mathcal{P}(\Pi)$, and then draw an action $\ra\sim\pi_{\text{hist}}$. To account for this, we use the following estimator,
\begin{align}
    \frac{\bar{\pi}(\ra|s)}{\pi_{\text{hist}}(\ra|s)}Q_{\pi_{\text{old}}}(s, \ra).\nonumber
\end{align}
Note that this requires that the value $\pi_{\text{hist}}(\ra|s)$ has also been inserted in the buffer. The expectation of the new estimator can be computed as
\begin{align}
    &\E_{\pi_{\text{hist}}\sim h, \ra\sim\pi_{\text{hist}}}\Big[ \frac{\bar{\pi}(\ra|s)}{\pi_{\text{hist}}(\ra|s)}Q_{\pi_{\text{old}}}(s, \ra)\Big] \nonumber\\
    &= \sum_{\pi_{\text{hist}}}h(\pi_{\text{hist}})\sum_{a\in\mathcal{A}}\pi_{\text{hist}}(a|s)\frac{\bar{\pi}(a|s)}{\pi_{\text{hist}}(a|s)}Q_{\pi_{\text{old}}}(s, a)\nonumber\\
    &= \sum_{a\in\mathcal{A}}\bar{\pi}(a|s)Q_{\pi_{\text{old}}}(s, a)\sum_{\pi_{\text{hist}}}h(\pi_{\text{hist}})\nonumber\\
    &= \sum_{a\in\mathcal{A}}\bar{\pi}(a|s)Q_{\pi_{\text{old}}}(s, a) = \E_{\ra\sim\bar{\pi}}\big[ Q_{\pi_{\text{old}}}(s, \ra) \big].\nonumber
\end{align}
Hence, the estimator has the desired mean.

%\section{Further Related Work}
%\label{sec:further_rel}
%The RL community has long worked on the development of theory which would result in theoretically-sound techniques. Q-learning \citep{Watkins1992} is an example of achievements along this thread, enabling implementation of GPI without the necessity of optimising the policy variable. This approach, however, limits the generality of policies that can be trained this way. A more general class of methods can be derived from the policy gradient theorem \citep{sutton:nips12}, where an agent optimises its policy parameters through gradient ascent. In practice, these techniques are quite unstable, allowing for large and risky updates along noisy gradient estimates \citep{kakade2002approximately}. 

%%%%%%%%%%%%%%%%%%%%%%%%%%%%%%%
%%%%%%%%%%%%%%%%%%%%%%%%%%%%%%%%%%%%%%%%%%%%%%%%%%%%%%%%%%%%%%%%%%%%%%%%%%%%%%%

\end{document}